\documentclass{article}

\usepackage[main, final]{neurips_2026}

\newcommand{\blfootnote}[1]{%
  \begingroup
  \renewcommand{\thefootnote}{}
  \footnotetext[0]{#1}
  \endgroup
}

\usepackage[utf8]{inputenc} 
\usepackage[T1]{fontenc} 
\usepackage{hyperref} 
\usepackage{url} 
\usepackage{booktabs} 
\usepackage{amsfonts} 
\usepackage{nicefrac} 
\usepackage{microtype} 
\usepackage{xcolor} 
\usepackage{subcaption}
\usepackage{float}

\usepackage{amsmath,amssymb,amsthm}
\usepackage{mathtools}
\usepackage{bm}
\usepackage{algorithm}
\usepackage{algorithmic}
\usepackage{natbib}
\usepackage{enumitem}
\usepackage{multirow}

\numberwithin{equation}{section}
\newtheorem{theorem}{Theorem}[section]
\newtheorem{lemma}[theorem]{Lemma}
\newtheorem{remark}[theorem]{Remark}
\newtheorem{corollary}[theorem]{Corollary}
\newtheorem{proposition}[theorem]{Proposition}
\newtheorem{assumption}{Assumption}
\newtheorem{definition}[theorem]{Definition}

\DeclareMathOperator*{\argmin}{argmin}
\DeclareMathOperator{\rank}{rank}
\DeclareMathOperator{\tr}{tr}

\newcommand{\E}{\mathbb{E}}
\newcommand{\N}{\mathbb{N}}
\newcommand{\R}{\mathbb{R}}
\newcommand{\B}{\mathbb{B}}

\title{How Finite-Rank Bottlenecks Shape the Low-Rank Adaptation Landscape}

\author{Long Nguyen-Chi$^{\star,1}$, Quynh Nguyen$^1$, Thanh Nguyen-Cung$^1$, Binh T. Nguyen$^{\star,1,2}$\\
$^{1}$College of Engineering \& Computer Science, VinUniversity\\
$^{2}$Center for AI Research, VinUniversity, Hanoi, Vietnam.
}

\begin{document}

\maketitle
\begin{abstract}
Low-rank adaptation (LoRA) has become the standard method for parameter-efficient fine-tuning of large pretrained models, yet theoretical explanations for why low-rank updates suffice remain incomplete. Most existing analyses rely on the Neural Tangent Kernel (NTK) regime, which linearizes the network around the pretrained weights. In this work, we take a different approach that requires no linearization: our key observation is that, once the backbone is frozen, 
the loss sees each adapted weight matrix only through a finite-rank linear map induced by the frozen activations. Trace-norm regularization then forces the optimizer to pick a low-rank matrix among all matrices that produce the same observed output. This yields global minimizers whose ranks are controlled by the rank of this map, which we call the \emph{bottleneck rank}, rather than the sample size. For multi-head attention, we obtain tighter rank upper bounds by exploiting two additional invariances: row-wise softmax is insensitive to additive row constants in the score matrix, and value updates are observable only after attention weighting and frozen output projection. For the factorized LoRA objective with weight decay, we prove a complementary result on the nonconvex side: after an arbitrarily small generic positive semidefinite perturbation, every first-order stationary point is rank-deficient once the LoRA rank exceeds a threshold determined by the bottleneck rank. Finally, we show that global minimizers of the perturbed objective are near-optimal for the original trace-norm problem, and we provide uniform generalization bounds governed by the spectral structure of the frozen features.
\blfootnote{$^\star$Corresponding author.}
\end{abstract}

\section{Introduction}\label{sec:intro}

Fine-tuning large pretrained models on downstream tasks is effective but expensive: updating all parameters of a billion-scale model requires substantial memory and compute.
Parameter-efficient fine-tuning (PEFT) solves this problem by adapting a large pretrained model while updating only a small subset of parameters. Some of the representative approaches include adapter modules~\citep{houlsby2019parameter}, prefix tuning~\citep{li2021prefix}, prompt tuning, and low-rank adaptation (LoRA, \citealt{hu2022lora}).
Among these, LoRA stands out as an attractive option because it keeps the pretrained backbone frozen and modifies selected linear maps by adding trainable low-rank updates.
More formally, for a weight matrix $W_0\in\R^{m\times n}$, LoRA parameterizes the adapted weight as
\[
W=W_0+T,\qquad T=UV^\top,\qquad U\in\R^{m\times r},V\in\R^{n\times r},
\]
The rank $r$ is typically chosen by the practitioner as a hyperparameter rather than something derived from the fine-tuning objective itself \citep{hu2022lora,schulman2025lora}.
A fundamental theoretical question is therefore: 
\begin{quote}
\emph{When does the regularized fine-tuning objective itself admit a low-rank global minimizer, and what controls the rank?}
\end{quote}

Most existing theory addresses nearby questions through linearization or simplified training dynamics.
In the NTK regime, \citet{jang2024lora} shows that full fine-tuning admits a low-rank solution and that sufficiently overparameterized LoRA eliminates spurious local minima.
\citet{afzal2025linearization} analyze fine-tuning through linearization under an explicit $\ell_2$ bias toward the pretrained model.
Beyond classical linearization, \citet{kimlora} study LoRA dynamics and prove a dichotomy between convergence to a low-rank global minimizer and convergence to a qualitatively different high-rank solution.
These works are complementary to ours, but they do not answer a more structural question: what property of the objective function makes low-rank optima available in the first place?

\paragraph{Contributions.}
Our answer starts with a simple observation: once the backbone is frozen, the loss does not observe an adapted matrix $T$ directly, but only through a finite-rank linear map determined by the frozen activations.
Concretely, if $H=[h_1,\ldots, h_N]$ is the matrix of frozen input features across the $N$ training samples, the empirical loss depends on $T$ only through the product $TH$.
We call $\rank(\Phi)$, where $\Phi(T)=TH$, the \emph{bottleneck rank} of the adapted block. The same principle applies blockwise to multiple adapted stages and attention heads.
Trace-norm regularization then selects a low-rank representative among all matrices that induce the same observed output.
This yields exact low-rank existence results without linearizing the network around the pretrained weights.
For multi-head attention \citep{vaswani2017attention}, the same principle becomes sharper because attention introduces additional invariances: row-wise softmax is unchanged by additive row constants in the score matrix, and value updates are observable only after attention weighting and the frozen output projection.

\begin{table}[h]
\centering
\small
\caption{Comparison with the closest related theoretical analyses of LoRA.}
\label{tab:comparison}
\setlength{\tabcolsep}{3pt}
\begin{tabular}
{@{}p{0.35\linewidth}p{0.22\linewidth}p{0.20\linewidth}p{0.20\linewidth}@{}}
\toprule
& \textbf{This work} & \cite{jang2024lora} & \cite{kimlora} \\
\midrule
Requires NTK linearization & No & Yes & No \\
Multilayer multi-head attention bound & Yes (Thm.~\ref{thm:rank-attn}) & No & No \\ 
Spurious local minima excluded & No (Prop.~\ref{prop:nogo-new}) & Yes (NTK) & Partial \\
Additional loss assumptions & Differentiability and stagewise $C^1$ regularity & NTK-regime conditions & Restricted strong convexity and smoothness \\
Generalization bounds & Yes & Yes & No \\
\bottomrule
\end{tabular}
\end{table}
Our contribution is summarized as follows.
\begin{itemize}[leftmargin=1.75em]
\item\textbf{Low-rank representer theorem for trace-norm fine-tuning.}
For a frozen network in which an adapted block $T$ is observed only through a finite-rank linear map $\Phi(T)$ of bottleneck rank $d_\Phi$, we prove that the trace-norm-regularized problem has a global minimizer whose rank is controlled by $d_\Phi$; see Theorem~\ref{thm:rank-single}.
The multi-stage extension is Theorem~\ref{thm:rank-L-stage}.

\item\textbf{Rank bounds from attention invariances.} 
For multi-head attention, query and key updates are observable only through centered score subspaces because row-wise softmax is invariant to row-wise additive constants.
Value updates are observable only after attention weighting and frozen output projection.
Combining these invariances with projector compression gives the optimizer-dependent subspace rank bounds in Theorem~\ref{thm:rank-attn}.

\item\textbf{Generic first-order rank deficiency for factorized LoRA.}
For the weight-decayed LoRA factorization, we prove that after a generic positive semidefinite perturbation, full-column-rank first-order stationary points are excluded whenever the Pataki--Barvinok rank threshold exceeds the dimension of the associated family of symmetric KKT matrices; see Theorem~\ref{thm:generic-rankdef-master}.
The multistage and attention consequence are Theorem~\ref{thm:rankdef-L-stage} and Corollary~\ref{cor:rankdef-L-attn}. 
This is a first-order exclusion theorem, not a no-spurious-local-minima theorem; Proposition~\ref{prop:nogo-new} gives an explicit limitation.

\item\textbf{Perturbation transfer and generalization.}
We show that global minimizers of the perturbed factorized objective are near-optimal for the original trace-norm-regularized problem when the perturbation is small; see Theorem~\ref{thm:transfer}.
\end{itemize}

\section{Related works}\label{sec:related-work}

\paragraph{Linearized and kernel views of fine-tuning.} The closest prior work to ours is \cite{jang2024lora}, which studies LoRA in the NTK regime. Using $K_{\mathrm{out}}$ for the number of output coordinates in the linearized training predictions, their low-rank existence theorem gives a solution satisfying $\frac{r(r+1)}{2}\le K_{\mathrm{out}}N$.
Thus the rank scale in that analysis is $r=\lfloor\frac{\sqrt{8K_{\mathrm{out}}N+1}-1}{2}\rfloor$, or $r=O(\sqrt{K_{\mathrm{out}}N})$ in the scalar-output convention.
Our results improve this comparison in two ways.
First, our bottleneck reduction is exact for frozen upstream computation and does not linearize the network around the pretrained weights.
Second, the optimizer's rank scale is $r\le\min\{K_{\mathrm{out}},N\}$ in Attention, and this is better than \cite{jang2024lora}'s bound. \cite{afzal2025linearization} also analyzed fine-tuning through linearization under an explicit $\ell_2$ bias toward the pretrained model and quantified when nonlinear fine-tuning stays close to the linearized problem.
In contrast, our reduction is exact once upstream computation is frozen; no approximation-by-linearization argument is needed.

\paragraph{Optimization theory beyond linearization.}
Recent works also study LoRA outside the lazy regime.
~\cite{kimlora} analyzes the LoRA loss landscape and proves that training converges either to a low-rank global minimizer or to a qualitatively different high-rank solution.
\cite{xu2025_aistats_dynamics} study gradient flow for LoRA in matrix factorization and show how initialization and singular-space alignment shape the final error.
\cite{liu2025_iclr_landscape} compares the landscapes of LoRA, GaLore, and full-rank adaptation through Polyak--\L{}ojasiewicz-type regions and spurious local minima.
\cite{zhang2024_icml_riemlora} analyzes a Riemannian preconditioner for LoRA in the infinite-width setting.
These works ask how LoRA training behaves.
Our question comes earlier: before one analyzes dynamics, when does the frozen regularized objective itself admit low-rank optima, and when must first-order stationary points of the factored problem be rank-deficient?

\paragraph{Attention and Transformer expressivity.}
Since transformers are built around attention \citep{vaswani2017attention}, we instantiate our framework for $Q/K/V$ projections in attention modules.
Orthogonally, a rich literature studies the expressive power and limitations of self-attention and transformers \citep{yun2020transformers,likhosherstov2021expressive,hahn2020theoretical,zhoualgorithms}.
Complementary work studies what functions low-rank adaptation can represent; \cite{zengexpressive} characterized the expressive power of low-rank adaptation in several settings.
For Transformer architectures, \cite{hu2025computational} studied computational limitations of LoRA fine-tuning.

\section{Problem setup}\label{sec:setup}

We study fine-tuning of a pretrained network in which all weights are frozen except for a specified collection of linear maps.
The full notation tables are collected in Appendix~\ref{sec:notation-full}; here we introduce only what is needed to state the main results.

\paragraph{Single adapted block.}
For a frozen weight matrix $W_0 \in\R^{m\times n}$, we write the adapted weight as $W=W_0+T$ with update $T \in\R^{m\times n}$.
Given training data $S_{\mathrm{data}}=\{(x_i, y_i)\}_{i=1}^N$, let $h_i \in\R^n$ denote the frozen input activation at this layer for sample $x_i$, and define the \emph{frozen feature matrix}
\[
H\coloneqq [h_1,\ldots, h_N] \in\R^{n \times N}.
\]
When the input $h_i$ is independent of the update $T$ and all other parameters are fixed, the network output on sample $x_i$ depends on $T$ only through $Th_i$ (Lemma~\ref{lem:frozen-suffix-multi}).

Stacking across samples, the empirical loss depends on $T$ only through
\(
\Phi(T)\coloneqq TH\in\R^{m\times N}.
\)
The map $\Phi$ is linear, and its range has dimension exactly $m\rank(H)\le m\min\{n,N\}$.

\begin{definition}[Bottleneck rank]\label{def:bottleneck-rank}
For a linear map $\Phi:\R^{m\times n}\to\mathcal{E}$ through which the loss observes an adapted block $T$, the bottleneck rank is
\[
d_\Phi\coloneqq\rank(\Phi)=\dim\mathcal{R}(\Phi).
\]
\end{definition}
When $\Phi(T)=TH$ as above, we have $d_\Phi=m\rank(H)$.
The bottleneck rank quantifies how much information the frozen network extracts from the update $T$: any two updates $T, T'$ with $TH=T'H$ produce identical training predictions.

\paragraph{Trace-norm regularization and the LoRA factorization.}
We consider the nuclear-norm-regularized (trace-norm-regularized) objective
\begin{equation}\label{eqn:obj-nuc}
\min_{T\in\R^{m\times n}}\mathcal{L}^{\mathrm{vanilla}}(T)+\lambda\|T\|_\star,\qquad\lambda>0,
\end{equation}
where $\mathcal{L}^{\mathrm{vanilla}}(T)$ is the empirical loss with update $T$ and $\|T\|_\star$ denotes the nuclear norm.
The factorized LoRA objective with weight decay at rank $r$ is
\[
\mathcal{L}^{\mathrm{lora}}_\lambda(U,V)\coloneqq\mathcal{L}^{\mathrm{vanilla}}(UV^\top)+\frac{\lambda}{2}\bigl(\|U\|_F^2+\|V\|_F^2\bigr),\quad U\in\R^{m\times r},V\in\R^{n \times r}.
\]
By the standard equivalence between nuclear-norm regularization and balanced weight decay~\citep{recht2010guaranteed}, the factorized problem has the same global minimum as~\eqref{eqn:obj-nuc} restricted to rank $\le r$:
\[
\min_{U,V} \mathcal{L}^{\mathrm{lora}}_\lambda(U,V)
=
\min_{\substack{T \in\R^{m\times n}\\ \rank(T)\le r}}
\Bigl\{\mathcal{L}^{\mathrm{vanilla}}(T)+\lambda\|T\|_\star\Bigr\}.
\]
Thus, whenever the unrestricted trace-norm problem~\eqref{eqn:obj-nuc} admits a global minimizer of rank $\le r$, the rank-$r$ LoRA factorization incurs no loss of optimality.

\paragraph{Attention mechanism}
For multi-head attention with head set $\mathcal{A}$, head $a \in\mathcal{A}$ produces adapted queries, keys, and values on sample $i$:
\[
\widetilde{Q}_i^{(a)}=(W_{Q,0}^{(a)}+T_Q^{(a)}) H_{Q,i},\quad
\widetilde{K}_i^{(a)}=(W_{K,0}^{(a)}+T_K^{(a)}) H_{K,i},\quad
\widetilde{V}_i^{(a)}=(W_{V,0}^{(a)}+T_V^{(a)}) H_{V,i},
\]
where $H_{S,i}$ is the frozen token matrix entering projection $S \in\{Q, K,V\}$ for sample $i$.
Queries and keys share output dimension $d_k$; values have output dimension $d_v$.
The attention matrix and block output are
\[
A_i^{(a)}=\mathrm{softmax}_{\mathrm{row}}\left(\frac{(\widetilde{Q}_i^{(a)})^\top \widetilde{K}_i^{(a)}}{\sqrt{d_k}}+C_i^M\right),
\qquad
O_i=\sum_{a \in\mathcal{A}} W_O^{(a)} \widetilde{V}_i^{(a)} (A_i^{(a)})^\top,
\]
where $C_i^M$ is an attention mask and $W_O^{(a)}$ is the frozen output-projection block for head $a$.
We write $H_S\coloneqq [H_{S,1},\ldots, H_{S,N}]$ for the sample-stacked feature matrices.
For sample $i$, let $\varrho_{Q,i}$ and $\varrho_{K,i}$ denote the numbers of query and key token positions, respectively, so that $A_i^{(a)}\in\R^{\varrho_{Q,i}\times \varrho_{K,i}}$.
In the stage-indexed notation used below, at stage $t$, these quantities are written as $\varrho_{t,Q,i}$ and $\varrho_{t,K,i}$.

\paragraph{Multiple adapted layers.}
When $L$ consecutive attention modules are adapted (indexed so that larger $t$ denotes a more upstream stage), the frozen features at stage $t$ may depend on the updates at stages $\tau>t$.
This creates a stagewise bottleneck structure, formalized as follows.

\begin{definition}[$L$-stage block-linear dependence]\label{def:L-stage-block-lin-meas}
Let $L \in\N$ and let $\mathcal{S}_t$ be finite index sets for $t \in [L]$.
Let $\mathcal{Z}_{L+1}$ be a singleton and let $\mathcal{Z}_t$ be finite-dimensional inner-product spaces for $t \in [L]$.
We say that a loss map $\mathcal{L}(\{T_{t,S}\}_{t \in [L], S \in\mathcal{S}_t})$ has \emph{$L$-stage block-linear dependence} if there exist maps
\[
\Phi_{t,S}:\R^{m_{t,S}\times n_{t,S}}\times \mathcal{Z}_{t+1} \to\mathcal{E}_{t,S},
\qquad
g_t:\Big(\prod_{S \in\mathcal{S}_t}\mathcal{E}_{t,S}\Big)\times \mathcal{Z}_{t+1} \to\mathcal{Z}_t
\quad (t \in [L]),
\]
and $g_0:\mathcal{Z}_1 \to\R$ such that:
\begin{enumerate}[leftmargin=1.75em]
\item for every fixed $z_{t+1} \in\mathcal{Z}_{t+1}$, the map $T \mapsto \Phi_{t,S}(T, z_{t+1})$ is linear;
\item if $z_{L+1}$ denotes the unique point of $\mathcal{Z}_{L+1}$ and
\[
z_t \coloneqq g_t\Bigl(\{\Phi_{t,S}(T_{t,S}, z_{t+1})\}_{S \in\mathcal{S}_t}, z_{t+1}\Bigr),
\quad t=L, L-1,\ldots, 1,
\]
then $\mathcal{L}(\{T_{t,S}\})=g_0(z_1)$.
\end{enumerate}
\end{definition}

The key feature is that each block $T_{t,S}$ enters stage $t$ through a linear map $\Phi_{t,S}$, but the stage transition maps $g_t$ may couple the blocks in an arbitrary (possibly nonconvex) way.
Stagewise affine shifts (e.g., the frozen term $W_{t,S,0} H_{t,S}(z_{t+1})$) are absorbed into the stage maps $g_t$.
For attention, the concrete map is $\Phi_{t,(a,S)}(T, z)=TH_{t,S}(z)$, where $H_{t,S}(z)$ is the frozen token matrix at stage $t$.
An output projection adapted alongside $Q/K/V$ receives inputs that depend on those same-stage updates.

\paragraph{Stagewise bottleneck ranks.}
For $z_{t+1} \in\mathcal{Z}_{t+1}^{\mathrm{reach}}$ (the set of reachable upstream states; see Definition~\ref{def:reach-z} in the appendix), define
\[
d_{t,S}(z_{t+1})\coloneqq\dim\mathcal{R}(\Phi_{t,S}(\cdot, z_{t+1})),
\qquad
d_{t,S}^{\max}\coloneqq\sup_{z_{t+1} \in\mathcal{Z}_{t+1}^{\mathrm{reach}}} d_{t,S}(z_{t+1}).
\]
The recursive \emph{upstream dimension budget} is
\[
\beta_{L+1} \coloneqq 0,\qquad
\beta_t \coloneqq\beta_{t+1}+\sum_{S \in\mathcal{S}_t} d_{t,S}^{\max},
\quad t=L, L-1,\ldots, 1.
\]
Intuitively, $\beta_t$ upper-bounds the intrinsic dimension of the stage-$t$ reachable state family.

\paragraph{Standing assumptions.}
We use two mild assumptions throughout, which are exactly the standard evaluation-mode condition.

\begin{assumption}[Well-posed objective]\label{ast:meas-wp-new-new}
The loss function $g$ (or $g_0$ in $L$-stage formulation, through which the empirical loss factors) is bounded below and lower semicontinuous, and all maps in the stagewise decomposition are continuous.
\end{assumption}

\begin{assumption}[Deterministic forward pass]\label{ast:frozen-suffix-new}
The forward pass is deterministic per sample (no Dropout, no BatchNorm).
\end{assumption}

\section{Main results}\label{sec:results-new}

All of our results follow from the same principle: freezing the backbone makes each adapted block observable only through a finite-rank linear map, and trace-norm regularization selects a low-rank representative among all updates that produce the same observations.

\subsection{Low-rank minimizers for a single frozen block}\label{sec:single-block}

\begin{theorem}[Low-rank global minimizers from bottleneck rank]\label{thm:rank-single}
Let Assumptions~\ref{ast:meas-wp-new-new} and~\ref{ast:frozen-suffix-new} hold.
Suppose the empirical loss depends on $T$ only through a linear map $\Phi(T)$ of bottleneck rank $d_\Phi$.
Then the trace-norm-regularized problem~\eqref{eqn:obj-nuc} admits a global minimizer.
Moreover, for every global minimizer $\tilde{T}$, there exists a global minimizer $T^\star$ with
\[
\Phi(T^\star)=\Phi(\tilde{T})\qquad\text{and}\qquad\frac{\rank(T^\star)\bigl(\rank(T^\star)+1\bigr)}{2}\le d_\Phi.
\]
Consequently,
\[
\rank(T^\star)\le\min\left\{m, n,\left\lfloor\frac{\sqrt{8 d_\Phi+1}-1}{2}\right\rfloor\right\}.
\]
\end{theorem}

\paragraph{Proof idea.}
Among all PSD lifts $Z \succeq 0$ with $\Phi(\bar{Z})=\Phi(\tilde{T})$, choose one that minimizes $\tr(Z)$.
Any extreme point of this spectrahedron satisfies the Barvinok--Pataki rank bound~\citep{pataki1998rank, barvinok2001remark}: its rank $r$ obeys $r(r+1)/2\le d_\Phi$.
The off-diagonal block $\bar{Z}$ recovers $T^\star$ with the same or smaller nuclear norm.
The full proof is in Appendix~\ref{sec:proof-single-block}.

\begin{remark}[Single-block rank scales]
For a single frozen block with $\Phi(T)=TH$, we have $d_\Phi\le m\rank(H)\le mN$.
When the observed output dimension is smaller than $\min\{K_{\mathrm{out}},\rank(H)\}\le\min\{K_{\mathrm{out}},N\}$, this has the same worst-case counting scale as $O(\sqrt{K_{\mathrm{out}}N})$ in \cite{jang2024lora}.
\end{remark}

\paragraph{Extension to $L$ consecutive adapted layers.}
The next theorem concerns multiple adapted stages, each potentially containing several heads and projections.
The same argument extends by stagewise induction: fix upstream stages, apply Theorem~\ref{thm:rank-single} to the conditional objective at each stage, and propagate downward.

\begin{theorem}[Low-rank global minimizers for $L$ stages]\label{thm:rank-L-stage}
Under Assumptions~\ref{ast:meas-wp-new-new} and~\ref{ast:frozen-suffix-new}, suppose the loss has $L$-stage block-linear dependence (Definition~\ref{def:L-stage-block-lin-meas}) and let all $\lambda_{t,S}>0$.
Then the trace-norm-regularized objective admits a global minimizer $\{T_{t,S}^\star\}$ satisfying, for every stage $t \in [L]$ and block $S \in\mathcal{S}_t$,
\[
\frac{r_{t,S}(r_{t,S}+1)}{2}\le d_{t,S}(z_{t+1}^\star)\le d_{t,S}^{\max},
\qquad r_{t,S}\coloneqq\rank(T_{t,S}^\star).
\]
\end{theorem}

The case $L=1$ recovers Theorem~\ref{thm:rank-single}.
For generic frozen linear layers, these bounds can still be loose because $d_{t,S}^{\max}$ may be large.
The next subsection shows that multi-head attention \emph{substantially tightens} the effective bottleneck.

\subsection{Rank bounds for multi-head attention}\label{sec:attn-bounds}

For attention modules, the bottleneck rank from Theorem~\ref{thm:rank-L-stage} can be refined using two invariances below that are specific to the attention mechanism.

\paragraph{Invariance 1: score centering.}
Row-wise softmax satisfies $\mathrm{softmax}_{\mathrm{row}}(M+C)=\mathrm{softmax}_{\mathrm{row}}(M P_{\mathbf{1}}^\perp+C)$, where $M,C\in\R^{q\times k}$ (allowing $-\infty$ entries in the mask $C$) and $P_{\mathbf{1}}^\perp=I_k-\frac{1}{k}\mathbf{1}_k\mathbf{1}_k^\top\in\R^{k\times k}$ (Lemma~\ref{lem:softmax-row-center}).
Thus, query/key scores are observed only after centering across key columns, reducing the key-token dimension from $\varrho_{K,i}$ to $\varrho_{K,i}-1$.
For queries and keys, the loss observes the update $T_Q^{(a)}$ only through the \emph{centered key subspace}, and similarly for $T_K^{(a)}$ through the query subspace.

\paragraph{Invariance 2: value-projection compression.}
Value updates are observable to the loss only after two compressions:
(i) attention-weighted pooling by $A_i^{(a)}$, which projects onto a subspace of dimension $\rank(A_i^{(a)})$; and
(ii) the frozen output projection $W_O^{(a)}$.

\paragraph{Projector compression principle.}
Both invariances are instances of the following general lemma: if $\Phi(T)=\Phi(P_{\mathrm{L}} T P_{\mathrm{R}})$ for orthogonal projectors $P_{\mathrm{L}},P_{\mathrm{R}}$, then $T^\star=P_{\mathrm{L}}\tilde{T}P_{\mathrm{R}}$ is an equally optimal global minimizer with $\rank(T^\star)\le\min\{\rank(P_{\mathrm{L}}),\rank(P_{\mathrm{R}})\}$ (Lemma~\ref{lem:proj-compression}).

Combining these ingredients gives per-head, per-projection rank certificates:

\begin{theorem}[Rank bounds for multi-head attention from invariances]\label{thm:rank-attn}
Let Assumptions~\ref{ast:meas-wp-new-new} and~\ref{ast:frozen-suffix-new} hold.
Consider $L$ consecutive multi-head attention modules with adapted blocks $\mathcal{S}_t \subseteq\mathcal{A}_t \times \{Q,K,V\}$, fixed output-projection blocks $W_{t,O}^{(a)}$, and regularization weights $\lambda_{t,S}^{(a)}>0$. Every attention row has at least one unmasked key.
Then the nuclear-norm-regularized objective admits a global minimizer.
Moreover, there exists a global minimizer $\{T_{t,S}^{(a)\star}\}$ that minimizes total rank $\sum_{t,a,S}\rank(T_{t,S}^{(a)})$ over the global minimizer set, and this minimizer satisfies:

\medskip
\noindent\textbf{Query bound:}
\[
\rank(T_{t,Q}^{(a)\star})\le\min\left\{d_{k,t}, \rank(H_{t,Q}^\star),\rank(H_{t,K}^\star),\sum_{i=1}^N(\varrho_{t,K,i}-1)\right\}\quad\forall\,(a,Q)\in\mathcal{S}_t.
\]

\noindent\textbf{Key bound:}
\[
\rank(T_{t,K}^{(a)\star})\le\min\left\{d_{k,t},\rank(H_{t,Q}^\star),\rank(H_{t,K}^\star),\sum_{i=1}^N(\varrho_{t,K,i}-1)\right\}\quad\forall\, (a, K)\in\mathcal{S}_t.
\]

\noindent\textbf{Value bound:}
\[
\rank(T_{t,V}^{(a)\star})\le\min\left\{d_{v,t},\rank(H_{t,V}^\star),\sum_{i=1}^N \rank(A_{t,i}^{(a)\star})\right\}\quad\forall\,(a,V)\in\mathcal{S}_t.
\]
Here, $H_{t,S}^\star$ and $A_{t,i}^{(a)\star}$ are the frozen features and attention matrices evaluated at the same global minimizer whose rank is being bounded.
\end{theorem}

\paragraph{Interpretation.} Our bounds have the following implications.
First, query and key ranks are controlled by the centered key/query subspace. The term $\sum_{i=1}^N(\varrho_{t,K,i}-1)$ arises from the score-centering invariance.
Second, value rank is controlled by attention sparsity.
The term $\sum_{i=1}^N\rank(A_{t,i}^{(a)\star})$ reflects the effective number of distinct attention patterns.
When attention is sparse or concentrated on a few dominant keys per query, $\rank(A_{t,i}^{(a)\star})$ is small, yielding a tight value-rank bound.
Third, the frozen output projection provides an unconditional value-rank bound. Since $\rank(T_{t,V}^{(a)\star})\le\rank(W_O^{(a)\top})\le d_{v,t}$, the value update can never exceed the per-head value dimension, regardless of the data.

\paragraph{Adapted output projections.}
The results aboved are extended for output projection in the next theorem.
The preceding rank certificates assumed fixed output-projection weights.
They extend structurally when one also adapts the full concatenated-head output projection, provided its input is treated as fixed only \emph{conditionally} on the other adapted blocks.

\begin{theorem}[Attention rank bounds with adapted output projections]\label{thm:rank-attn-adapted-output}
In the setting of Theorem~\ref{thm:rank-attn}, suppose that at each selected stage $t$ the full output projection $W_{t,O,0}\in\R^{d_{\mathrm{out},t}\times|\mathcal{A}_t|d_{v,t}}$ also receives an update $T_{t,O}$ of the same shape with nuclear-norm weight $\lambda_{t,O}>0$.
Then the nuclear-norm-regularized objective admits a global minimizer that minimizes the total rank of all adapted $Q/K/V/O$ blocks over the global minimizer set.
At every such minimizer, the query and key bounds of Theorem~\ref{thm:rank-attn} hold unchanged, and
\begin{align*}
\rank(T_{t,V}^{(a)\star})&\le\min\left\{d_{v,t},\rank(H_{t,V}^\star),\sum_{i=1}^N\rank(A_{t,i}^{(a)\star})\right\},\\
\rank(T_{t,O}^\star)&\le\min\left\{d_{\mathrm{out},t},|\mathcal{A}_t|d_{v,t},\sum_{i=1}^N\varrho_{t,Q,i}\right\},
\end{align*}
where all features and attention matrices are evaluated at this minimizer.
\end{theorem}

A value update $T_{t,V}^{(a)}$ enters the stage output only after left multiplication by the head-$a$ column block $\widetilde W_{t,O}^{(a)}\in\R^{d_{\mathrm{out},t}\times d_{v,t}}$ of the adapted output projection $W_{t,O,0}+T_{t,O}$.
Adapting the output projection changes which directions of the $d_{v,t}$-dimensional head output are retained, but it cannot enlarge that output.
The update $T_{t,O}$ differs from the $Q/K/V$ updates because its input, the attention-weighted values, depends on the same-stage $Q/K/V$ updates.
The proof therefore compresses each block conditionally on the other blocks instead of using a joint linear map, and for the same reason Corollary~\ref{cor:rankdef-L-attn} does not apply to $T_{t,O}$.

\subsection{Generic rank deficiency of first-order stationary points}\label{sec:rankdef}

The results in Sections~\ref{sec:single-block}--\ref{sec:attn-bounds} concern \emph{global} minimizers of the convex trace-norm objective.
In practice, LoRA solves the \emph{nonconvex} factorized problem $\mathcal{L}^{\mathrm{lora}}_\lambda(U,V)$ via gradient descent, which may converge to stationary points that are not global minimizers.
We now show that the bottleneck structure also constrains the rank at stationary points of the factorized problem.

\begin{assumption}[Differentiability for stationarity]\label{ast:stationarity-diff}
The induced empirical loss $\mathcal{L}$ is differentiable with respect to every adapted block.
\end{assumption}
\begin{assumption}[Stagewise $C^1$ regularity for stagewise linear maps]\label{ast:multistage-c1}
In Definition~\ref{def:L-stage-block-lin-meas}, for every stage $t\in[L]$, every block $S\in\mathcal{S}_t$, every basis vector $e_j$ of $\mathcal{E}_{t,S}$, and every basis matrix $B_k$ of $\R^{m_{t,S}\times n_{t,S}}$, the maps $z_{t+1}\mapsto\Phi_{t,S}(\cdot,z_{t+1})^\star(e_j),z_{t+1}\mapsto\Phi_{t,S}(B_k,z_{t+1})$ are $C^1$ on an open neighborhood of $\mathcal{Z}_{t+1}^{\mathrm{reach}}$.
Here, $\Phi_{t,S}^\star$ is the adjoint map of $\Phi_{t,S}$.
Moreover, each transition map $g_t$ is $C^1$ on an open neighborhood of the reachable set of its arguments.
\end{assumption}

Assumption~\ref{ast:multistage-c1} is a regularity assumption on how the moving linear observation maps depend on upstream states, and common evaluation-mode attention blocks satisfy this condition on bounded reachable sets.

\paragraph{Why a perturbation is needed.}
The KKT conditions at stationarity gives $(\mathcal{B}_{m,n}(G)+\lambda I_{m+n})R=0$ with $G=\nabla_T\mathcal{L}$ and $\mathcal{B}_{m,n}(G)=\bigl[\begin{smallmatrix}0&G\\G^\top&0\end{smallmatrix}\bigr]$.
Full column rank of $R$ forces an $r$-dimensional nullspace, whereas bottleneck structure restricts $G$ to the low-dimensional range $\mathcal{R}(\Phi^\star)$ (Lemma~\ref{lem:grad-range-L-stage}).
A generic PSD perturbation makes this low-dimensional symmetric-matrix family avoid such a high-dimensional nullspace.

\begin{definition}[Perturbed LoRA objective]\label{def:perturbed-lora-new}
For perturbation $P\in\mathbb{S}_+^{m+n}$, define
\begin{equation*}
\mathcal{L}^{\mathrm{lora}}_{\lambda,P}(U,V)
\coloneqq
\mathcal{L}^{\mathrm{vanilla}}(UV^\top)+\frac{\lambda}{2}\bigl(\|U\|_F^2+\|V\|_F^2\bigr)+\frac{1}{2}\langle P, RR^\top\rangle.
\end{equation*}
\end{definition}

\begin{theorem}[Generic exclusion of full-rank stationary points]\label{thm:generic-rankdef-master}
Fix $d\ge1$, $r\ge1$, $\lambda>0$, and let $\mathcal{G}\subseteq\mathbb{S}^d$ be contained in a finite union of $C^1$ images of dimension at most $p$.
If $\frac{r(r+1)}{2}>p$,
then for Lebesgue-a.e.\ $P\in\mathbb{S}_+^d$, every matrix $R \in\R^{d\times r}$ satisfying $(G+\lambda I+P)R=0$ for some $G\in\mathcal{G}$ has $\rank(R)<r$.
\end{theorem}

\paragraph{Proof idea.}
For each $s\ge r$, the set of symmetric matrices with nullspace dimension exactly $s$ is a smooth embedded submanifold of $\mathbb{S}^d$ of codimension $s(s+1)/2$.
Therefore, the set of symmetric matrices with nullspace dimension at least $r$ is a finite union of such manifolds, not generally a single submanifold.
The family $\mathcal{G}+\lambda I$ has dimension at most $p$.
When $r(r+1)/2>p$, the translated family $\mathcal{G}+\lambda I+P$ generically misses this submanifold by a dimension-counting argument (Lemma~\ref{lem:measure-zero-sum}).

\paragraph{Application to attention.}
Using the stagewise gradient-range characterization (Lemma~\ref{lem:reach-move-param-L}), the family $\mathcal{G}_{t,S}$ for each attention block is $C^1$-parameterizable with dimension at most $\beta_{t+1}+d_{t,S}^{\max}$.
This yields:

\begin{corollary}[Generic rank deficiency for $L$ attention modules]\label{cor:rankdef-L-attn}
Under Assumptions~\ref{ast:meas-wp-new-new},\ref{ast:frozen-suffix-new}, and mild differentiability conditions (Assumptions~\ref{ast:stationarity-diff} and~\ref{ast:multistage-c1}), consider $L$ consecutive attention modules.
For each adapted block $\alpha=(a,S)\in\mathcal{S}_t$, let $r_{H,t,\alpha}^{\max}\coloneqq\sup_{z_{t+1}}\rank(H_{t,S}(z_{t+1}))$.
If
\[
\frac{r_{t,\alpha}(r_{t,\alpha}+1)}{2}>\beta_{t+1}+m_{t,\alpha}r_{H,t,\alpha}^{\max}\qquad\forall\,t\in[L],\alpha\in\mathcal{S}_t,
\]
then for Lebesgue-a.e.\ PSD perturbations $\{P_{t,\alpha}\}$, every first-order stationary point of the perturbed LoRA objective satisfies $\rank(R_{t,\alpha})<r_{t,\alpha}$.
\end{corollary}

\paragraph{Limitation: spurious local minima persist.}
This is a first-order \emph{exclusion} theorem, not a no-spurious-local-minima result.
Indeed, spurious local minima can occur even for arbitrarily overparameterized LoRA:

\begin{proposition}[Spurious local minima for overparameterized LoRA]\label{prop:nogo-new}
For any rank $r\ge 2$, any $\lambda>0$, and any $\varepsilon \in (0,1)$, there exists a smooth nonconvex loss $g$ such that for every $P\in\mathbb{S}_+^2$ with $\|P\|_F \le\varepsilon$, the perturbed rank-$r$ LoRA objective has a strict local minimum at $R=0$ that is not global.
\end{proposition}

This clarifies the scope of the generic rank-deficiency result: it guarantees that stationary points are rank-deficient (the LoRA factorization does not waste its rank budget at stationary points), but it does not guarantee that all such points are globally optimal.

\subsection{Perturbation transfer and generalization}\label{sec:transfer-gen}

\paragraph{Near-optimality of the perturbed factorization.}
The PSD perturbation in Theorem~\ref{thm:generic-rankdef-master} can be taken arbitrarily small.
The following result shows that global minimizers of the perturbed factorized objective remain near-optimal for the original trace-norm problem:

\begin{theorem}[Perturbation-to-unperturbed transfer]\label{thm:transfer}
Let $\mathcal{I}$ be any finite index set of adapted blocks.
For each $j\in\mathcal{I}$, fix a factorization rank $r_j$ and a regularization parameter $\lambda_j>0$.
Assume that the trace-norm-regularized problem admits a global minimizer $\{T_j^\star\}$ satisfying $\rank(T_j^\star)\le r_j\forall\,j\in\mathcal{I}$, and let $R_j^\star\coloneqq\bigl[\begin{smallmatrix}U_j^\star\\V_j^\star\end{smallmatrix}\bigr]$ for a balanced factorization of $T_j^\star$.
Let $\{(\widehat{U}_j,\widehat{V}_j)\}$ be a global minimizer of the perturbed LoRA objective with parameters $\lambda_j,P_j$ (Definition~\ref{def:perturbed-lora-new}).
Then
\[
\mathcal{L}(\{\widehat{T}_j\})+\sum_j\lambda_j\|\widehat{T}_j\|_\star\le\min_{\{T_j\}}\left\{\mathcal{L}(\{T_j\})+\sum_j\lambda_j\|T_j\|_\star\right\}+\frac{1}{2}\sum_j\langle P_j,R_j^\star R_j^{\star\top}\rangle.
\]
In particular, if $\|P_j\|_2\le\varepsilon_j$, the gap is at most $\sum_j\varepsilon_j\|T_j^\star\|_\star$.
\end{theorem}
Our assumptions satisfy that the trace-norm-regularized problem admits a global minimizer, and this theorem can be applied to a single block index set $\mathcal{S}$, a stagewise index set $\mathcal{S}_t$, and the flattened attention index set $\widetilde{\mathcal{S}}\coloneqq\{(t,a,S):t\in[L],(a,S)\in\mathcal{S}_t\}$.

\paragraph{Generalization.}
Uniform generalization bounds for trace-norm-bounded updates in frozen networks are governed by $\E[\|EH^\top\|_2]$, where $E$ is a Rademacher sign matrix.
The key identity is the trace/spectral duality:
\[
\sup_{\|T\|_\star\le D} \langle T, EH^\top\rangle_F=D \|EH^\top\|_2,
\]
which reduces the Rademacher complexity of the loss class to the spectral norm of random-sign-weighted frozen features.
The full single-block bound (Theorem~\ref{thm:gen-full-single}) and the $L$-stage attention extension (Theorem~\ref{thm:gen-full-L-attn}) are stated and proved in Appendix~\ref{sec:gen-full}.

\section{Experiment}\label{sec:experiment}

\begin{figure}[ht]
\centering
\begin{subfigure}[b]{0.9\textwidth}
\centering
\includegraphics[width=\linewidth]{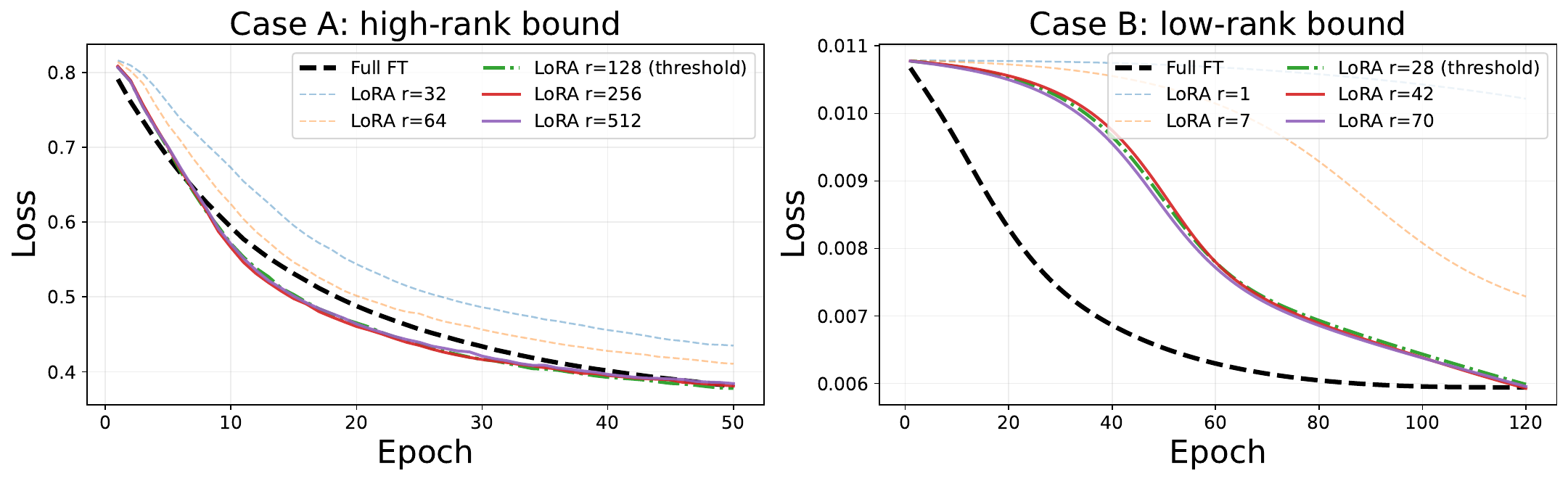}
\caption{MNLI}
\label{fig:MNLI}
\end{subfigure}
\hfill
\begin{subfigure}[b]{0.9\textwidth}
\centering
\includegraphics[width=\linewidth]{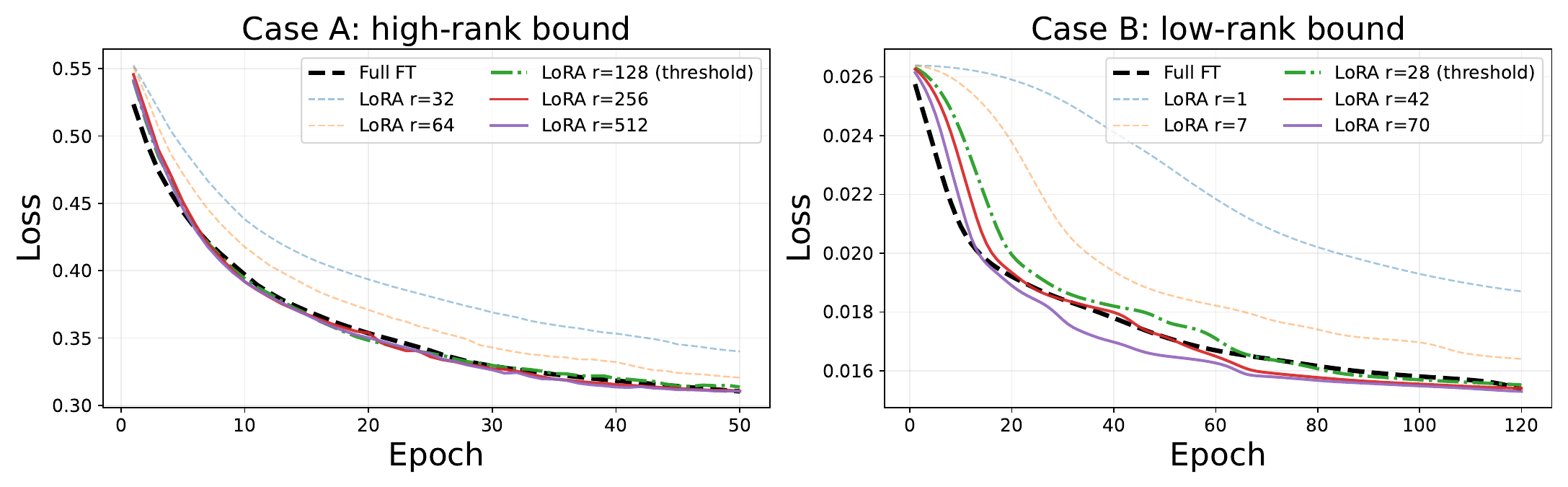}
\caption{SST-2}
\label{fig:SST2}
\end{subfigure}
\hfill
\begin{subfigure}[b]{0.9\textwidth}
\centering
\includegraphics[width=\linewidth]{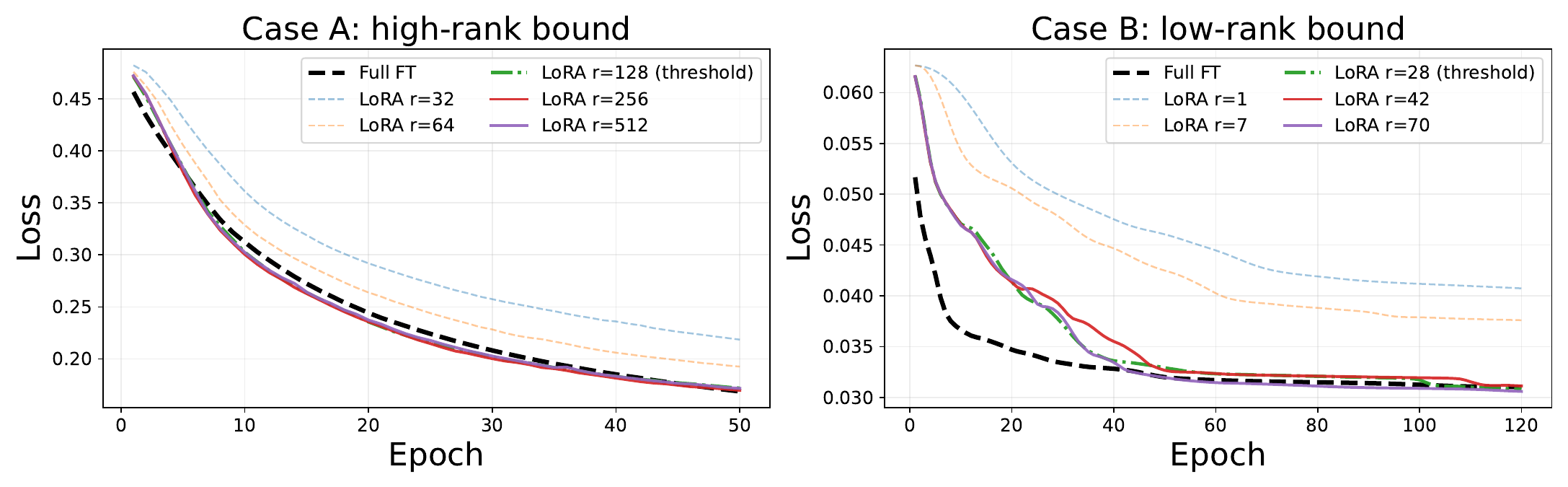}
\caption{RTE}
\label{fig:RTE}
\end{subfigure}
\caption{Training loss trajectories on GLUE classification tasks (MNLI, SST-2, and RTE) using Qwen2.5-1.5B across various LoRA ranks. The results demonstrate a clear threshold effect: ranks below the estimated threshold exhibit distinct, suboptimal convergence trajectories, whereas ranks at or above the threshold tightly cluster and converge to nearly indistinguishable final losses, closely matching the convergence point of full fine-tuning.}
\label{fig:classification_results}
\end{figure}

We test the predicted rank threshold on SST-2, MNLI, and RTE~\citep{wang2018glue} using Qwen2.5-1.5B~\citep{yang2025qwen3}.
Only one query-projection block in one head of layers 7 and 8 is adapted; after a classification-head warm-up, both the backbone and classifier are frozen.
The threshold is the smallest rank at which the upstream and downstream rank proxies both saturate.

\paragraph{Two rank-bound regimes.}
For the selected query blocks, Theorem~\ref{thm:rank-attn} gives the upper bound
\[
\rank(T_{t,Q}^{(a)\star})\le\min\left\{d_{k,t}, \rank(H_{t,Q}^\star),\rank(H_{t,K}^\star),\sum_{i=1}^N(\varrho_{t,K,i}-1)\right\}\quad\forall\,(a,Q)\in\mathcal{S}_t.
\]
We evaluate in two regimes.
We test a many-example regime where this estimate reaches the head cap $d_{k,t}$ and a few-example regime where it is strictly smaller.
For Qwen2.5-1.5B, $d_{k,t}=128$.

\paragraph{Results.}
Figure~\ref{fig:classification_results} shows a consistent threshold effect: ranks below the estimate follow distinct, worse loss trajectories, while ranks at or above it nearly overlap and approach the full fine-tuning loss.
This is consistent with the high-rank LoRA/full-fine-tuning alignment observed by \citet{jang2024lora}.
We also observe that high-rank LoRA runs converge to almost the same point as the full fine-tuning baseline, although their training dynamics differ, consistent with the related phenomenon observed by \citet{jang2024lora}.

\section{Conclusion}\label{sec:conclusion}

In this paper, we identify that LoRA operates via a finite-rank mechanism where frozen layers project updates into low-dimensional subspaces.
Through this, trace-norm regularization selects a low-rank representative among all updates producing the same observed outputs.
This yields global low-rank minimizers without NTK linearization and gives rank bounds controlled by bottleneck dimensions rather than by the ambient number of trainable parameters.

For attention modules, we showed that the observable subspaces are smaller than the naive frozen-feature image, row-wise softmax removes one key-column direction per sample, and value updates are observable only after attention-weighted pooling and the frozen output projection.
These invariances give per-head rank certificates for query, key, and value updates at rank-minimal global minimizers.
We also prove that small perturbations exclude full-rank stationary points in nonconvex LoRA while maintaining near-optimality.

\paragraph{Limitations.}
First, our rank theorems certify existence for trace-norm objectives, not convergence of gradient-based LoRA.
Second, we can't explicitly calculate bounds in Theorem~\ref{thm:rank-attn}, which are based on global solutions.
Finally, the attention bounds are a posteriori, and the generic rank-deficiency theorem is only a first-order statement for perturbed objectives; Proposition~\ref{prop:nogo-new} shows that spurious local minima can remain.

\section*{Acknowledgments} The authors would like to thank anonymous reviewers and the Area Chair for their useful feedback to improve our work. This work is supported by grant VUNI.2526.CAIR.04 and Seed Fund 10000225 CECS.

\bibliography{bib}
\bibliographystyle{apalike}
\newpage
\appendix

\section{Full list of notations and auxiliary definitions}\label{sec:notation-full}

We first summarize the key notation used throughout the paper. Table~\ref{tab:notation-lora} summarizes the basic symbols for the LoRA fine-tuning setup, and Table~\ref{tab:notation-attn} lists the additional notation used for attention layers and multi-stage linear mappings.

\begin{table}[h!]
\centering
\begin{tabular}{lll}
\hline
\textbf{Symbol} & \textbf{Type / Shape} & \textbf{Description}\\
\hline

$n$ & scalar & Input dimension \\
$m$ & scalar & Output dimension \\
\hline 

$W_0$ & $\R^{m\times n}$ & Pretrained (frozen) weight matrix \\
$W$ & $\R^{m\times n}$ & Adapted weight matrix ($W=W_0+T$)\\
$T$ & $\R^{m\times n}$ & Low-rank update matrix \\
$U$ & $\R^{m\times r}$ & Left low-rank factor\\
$V$ & $\R^{n \times r}$ & Right low-rank factor\\
$r$ & scalar & Rank of LoRA update ($r\ll \min(m,n)$)\\

$f_W$ & function & Network with frozen weights except adapted blocks \\

\hline
$S_{\mathrm{data}}$ & $\{(x_i,y_i)\}_{i=1}^N$ & Training dataset \\
$N$ & scalar & Number of training samples \\
$x_i$ & input & Input sample \\
$y_i$ & output & Target label \\

$h(x)$ & $\R^n$ & Frozen input activation to adapted layer\\
$h_i$ & $\R^n$ & Activation for sample $x_i$ ($h_i=h(x_i)$)\\
$H$ & $\R^{n \times N}$ & Stacked feature matrix $[h_1,\dots,h_N]$ \\

\hline
$\mathcal{S}$ & $\{Q,K,V\}$ & Attention projection blocks \\
$H_S(x)$ & matrix & Frozen token matrix for block $S$ \\
$H_{S,i}$ & matrix & Token matrix for sample $x_i$ \\

\hline
$W_S$ & matrix & Adapted weight for block $S$ ($W_S=W_{S,0}+T_S$)\\
$W_{t,S}$ & matrix & Weight at stage $t$ for block $S$ \\
$T_{t,S}$ & matrix & Update at stage $t$ \\

$H_{t,S}(x;\cdot)$ & matrix & Token matrix entering block $S$ at stage $t$ \\
$H_{t,S,i}$ & matrix & Same for sample $x_i$ \\

\hline
$\mathcal{L}$ & function & Loss function \\
$\Phi,\Phi_S$ & operator & Linear maps \\
$g$ & function & Scalar function in loss decomposition \\
$s, s_S$ & vector & Affine shift(s)\\
\hline

\end{tabular}
\caption{Summary of notation and parameters for LoRA fine-tuning setup} \label{tab:notation-lora}
\end{table}

\begin{table}[h!]
\centering
\begin{tabular}{lll}
\hline
\textbf{Symbol} & \textbf{Type / Shape} & \textbf{Description}\\
\hline
$H_{Q,i}, H_{K,i}, H_{V,i}$ & matrices & Token matrices for query, key, value (sample $i$)\\
$H_Q, H_K, H_V$ & stacked matrices & $[H_{S,1},\dots,H_{S,N}]$ for $S \in\{Q,K,V\}$ \\

$d_k$ & scalar & Query/key head dimension \\
$d_v$ & scalar & Value head dimension \\
$m_Q, m_K$ & scalar & Output dims of $Q,K$ projections ($= d_k$)\\

\hline
$W_{S,0}$ & matrix & Frozen projection for $S \in\{Q,K,V\}$ \\
$T_S$ & matrix & Trainable LoRA update for projection $S$ \\

$\widetilde{Q}_i$ & matrix & $(W_{Q,0}+T_Q)H_{Q,i}$ (adapted queries)\\
$\widetilde{K}_i$ & matrix & $(W_{K,0}+T_K)H_{K,i}$ (adapted keys)\\
$\widetilde{V}_i$ & matrix & $(W_{V,0}+T_V)H_{V,i}$ (adapted values)\\

\hline
$A_i$ & $\R^{\varrho_{Q,i}\times \varrho_{K,i}}$ & Attention matrix (row-softmax)\\
$C_i^M$ & matrix & Attention mask (with $-\infty$ entries allowed)\\
$O_i$ & matrix & Attention output: $W_O \widetilde{V}_i A_i^\top$ \\
$\varrho_{Q,i},\varrho_{K,i}$ & scalars & Number of query/key tokens for sample $i$ \\
$\varrho_{t,Q,i},\varrho_{t,K,i}$ & scalars & Number of query/key token positions at stage $t$ for sample $i$ \\

\hline
$\mathcal{A}$ & set & Set of attention heads \\
$a \in\mathcal{A}$ & index & Head index \\

$\widetilde{Q}_i^{(a)},\widetilde{K}_i^{(a)},\widetilde{V}_i^{(a)}$ & matrices & Head-specific projections \\
$A_i^{(a)}$ & matrix & Head-specific attention matrix \\
$W_O^{(a)}$ & matrix & Output projection block for head $a$ \\

$O_i$ & matrix & Multi-head output: $\sum_{a} W_O^{(a)} \widetilde{V}_i^{(a)} (A_i^{(a)})^\top$ \\

\hline
$\mathcal{A}_t$ & set & Head set at stage $t$ \\
$d_{k,t}, d_{v,t}$ & scalars & Head dimensions at stage $t$ \\

$H_{t,S,i}$ & matrix & Token matrix at stage $t$, block $S$, sample $i$ \\
$H_{t,S}$ & stacked matrix & $[H_{t,S,1},\dots,H_{t,S,N}]$ \\

\hline
$z_t$ & element of $\mathcal{Z}_t$ & Stage-$t$ state in model \\
$\alpha=(a,S)$ & index & Block index: head $a$, projection $S \in\{Q,K,V\}$ \\
$\Phi_\alpha(T_\alpha)$ & map & $T_\alpha H_S$ \\
$\Phi_{t,\alpha}(T,z)$ & map & Stagewise linear map: $T_{t,\alpha} H_{t,S}(z)$ \\

\hline
\end{tabular}
\caption{Summary of attention notation and multi-stage linear mapping}\label{tab:notation-attn}
\end{table}

For a single block, we also write
\[
F_{\mathrm{primal}}(T)\coloneqq g(\Phi(T)+s)+\lambda\|T\|_\star,F_{\mathrm{lift}}(Z)\coloneqq g(\Phi(\bar{Z})+s)+\frac{\lambda}{2}\tr(Z),Z\succeq0.
\]

\paragraph{NTK linearization as a special case.}
Our setting also contains the NTK/lazy regime as a special case \citep{jacot2018neural, lee2019wide}.
Let $\Delta\theta\in\R^p$ denote the trainable parameter update; if the adapted parameter is a matrix, one may take $\Delta\theta=\operatorname{vec}(T)$.
The first-order Taylor model around $\theta_0$ is
\[
f(x;\theta_0+\Delta\theta)\approx f(x;\theta_0)+J_{\theta_0}(x)\Delta\theta,
\]
where $J_{\theta_0}(x)\in\R^{K\times p}$ is the Jacobian of the logits with respect to the trainable parameters at $\theta_0$.
Given training data $\{(x_i,y_i)\}_{i=1}^N$, define
\[
\Phi_{\mathrm{NTK}}(\Delta\theta)\coloneqq\bigl(J_{\theta_0}(x_1)\Delta\theta,\ldots,J_{\theta_0}(x_N)\Delta\theta\bigr)\in(\R^K)^N
\]
and
\[
s_{\mathrm{NTK}}\coloneqq\bigl(f(x_1;\theta_0),\ldots,f(x_N;\theta_0)\bigr)\in(\R^K)^N.
\]
Then any empirical loss depending only on the vector of training predictions can be written as
\[
\mathcal{L}_{\mathrm{NTK}}(\Delta\theta)=g\bigl(\Phi_{\mathrm{NTK}}(\Delta\theta)+s_{\mathrm{NTK}}\bigr)
\]
for a suitable map $g:(\R^K)^N\to\R$.
For a sample-wise loss, one may take
\[
g(z_1,\ldots,z_N)\coloneqq\frac1N\sum_{i=1}^N\ell(z_i,y_i).
\]
In the scalar-output case $K=1$, this reduces to a map into $\R^N$. We now introduce several auxiliary concepts used later in the analysis.

\begin{definition}[$C^1$-parameterizable subsets]\label{def:c1-param}
Let $\mathcal{V}$ be a finite-dimensional real inner-product space and let $\mathcal{G}\subseteq\mathcal{V}$.
We say that $\mathcal{G}$ is $C^1$-parameterizable with dimension at most $p$ if there exist finitely many open sets $U_j\subseteq\R^{p_j}$ with $p_j\le p$ and $C^1$ maps $f_j:U_j\to\mathcal{V}$ such that
\[
\mathcal{G}\subseteq\bigcup_{j=1}^Jf_j(U_j).
\]
\end{definition}

\begin{definition}[Reachable sets]\label{def:reach-z}
For $L$-stage block-linear dependence, set $\mathcal{Z}_{L+1}^{\mathrm{reach}}\coloneqq\mathcal{Z}_{L+1}$.
For $t=L,L-1,\ldots,1$, define the reachable stage-$t$ linear-output set
\[
\mathcal{U}_t^{\mathrm{reach}}\coloneqq\left\{\{\Phi_{t,S}(T_{t,S},z_{t+1})\}_{S\in\mathcal{S}_t}:z_{t+1}\in\mathcal{Z}_{t+1}^{\mathrm{reach}},T_{t,S}\in\R^{m_{t,S}\times n_{t,S}}\right\}\subseteq\prod_{S\in\mathcal{S}_t}\mathcal{E}_{t,S},
\]
and the reachable stage-$t$ state set
\[
\mathcal{Z}_t^{\mathrm{reach}}\coloneqq\left\{g_t\Bigl(\{u_{t,S}\}_{S\in\mathcal{S}_t},z_{t+1}\Bigr):z_{t+1}\in\mathcal{Z}_{t+1}^{\mathrm{reach}},u_{t,S}\in\mathcal{R}(\Phi_{t,S}(\cdot,z_{t+1}))\right\}.
\]
\end{definition}

\paragraph{Population and empirical risk.}
For a single update $T\in\R^{m\times n}$, define
\[
L(T)\coloneqq\E_{(x,y)\sim\mathcal{P}}[\ell(f_{W_0+T}(x),y)],\quad \mathcal{L}^{\mathrm{vanilla}}_{S_{\mathrm{data}}}(T)\coloneqq\frac1N\sum_{i=1}^N\ell(f_{W_0+T}(x_i),y_i).
\]
For multiple adapted matrices $\{T_S\}$, we use the same notation blockwise.
When the sample is fixed, we write $\mathcal{L}^{\mathrm{vanilla}}$ for $\mathcal{L}^{\mathrm{vanilla}}_{S_{\mathrm{data}}}$.

For $D\ge0$, define the deterministic activation ball
\[
\mathcal{Z}_{\mathrm{act}}(D,R_h)\coloneqq\left\{z\in\R^m:\|z\|_2\le(\|W_0\|_2+D)R_h\right\}.
\]
To state our generalization bounds, we also recall the empirical Rademacher complexity.
\paragraph{Empirical Rademacher complexity.}
Let $\mathcal{F}$ be a class of functions $f:\mathcal{X}\times\mathcal{Y}\to\R$.
The (conditional) empirical Rademacher complexity of $\mathcal{F}$ on $S_{\mathrm{data}}$ is
\[
\mathfrak{R}_{S_{\mathrm{data}}}(\mathcal{F})\coloneqq\E_\varepsilon\left[\sup_{f\in\mathcal{F}}\frac1N\sum_{i=1}^N\varepsilon_if(x_i,y_i)\mid S_{\mathrm{data}}\right],
\]
where $\varepsilon_1,\dots,\varepsilon_N$ are i.i.d.\ uniform signs in $\{-1,1\}$.
For $D\ge0$, define the trace-norm ball
\[
\B_\star(D)\coloneqq\{T\in\R^{m\times n}:\|T\|_\star\le D\},
\]
and define the loss class induced by trace-norm-bounded updates
\[
\mathcal{F}_D\coloneqq\{(x,y)\mapsto\ell(f_{W_0+T}(x),y):T\in\B_\star(D)\}.
\]

\paragraph{Miscellaneous.}
For $Z\in\mathbb{S}_+^{m+n}$ write $Z=\begin{bmatrix}Z_{11}&Z_{12}\\Z_{21}&Z_{22}\end{bmatrix}$ with $Z_{12}\in\R^{m\times n}$ and denote $\bar{Z}\coloneqq Z_{12}$.
For sets $A,B$ in a vector space, the Minkowski sum is $A+B\coloneqq\{a+b:a\in A,b\in B\}$.
We use $\operatorname{blkdiag}$ for a block diagonal matrix.
For a matrix $M$, let $P_M$ denote the orthogonal projector onto the column space $\mathcal{R}(M)$.
If $M=0$, then $P_M=0$.

\section{Generalization bounds}\label{sec:gen-full}

We now state the assumptions used to derive the generalization bounds in this section. 

\begin{assumption}[Bounded and Lipschitz loss]\label{ast:bounded-lipschitz-loss}
The loss satisfies $0\le\ell(\cdot,y)\le B$ and is $L_\ell$-Lipschitz ($L_\ell>0$) with respect to $\|\cdot\|_2$ for every $y$.
\end{assumption}

\begin{assumption}[Almost-sure bounded frozen activations]\label{ast:bounded-activations}
For $(x,y)\sim\mathcal{P}$, the frozen input activation to the single adapted block satisfies $\|h(x)\|_2\le R_h$ almost surely.
In the attention setting, for each $S\in\{Q,K,V\}$, the frozen token matrix satisfies $\|H_S(x)\|_F\le R_{H,S}$ almost surely.
Consequently, on every realized sample,
\[
\|h_i\|_2\le R_h, \qquad \|H_{S,i}\|_F\le R_{H,S}\ (i\in[N],\ S\in\{Q,K,V\}).
\]
\end{assumption}

\begin{assumption}[Inter-stage token Lipschitzness for $L$ consecutive attention modules]\label{ast:L-attn-token-lipschitz}
Fix trace-norm radii $(D_{t,S})_{t\in[L],S\in\mathcal{S}_t}$.
For each $t\in\{1,\ldots,L-1\}$, assume there exists a finite constant $L_{H,t}(D_{>t})$ such that for all $i\in[N]$ and all two upstream update collections $\{T_{\tau,S}\}_{\tau>t}$ and $\{T'_{\tau,S}\}_{\tau>t}$ satisfying $\|T_{\tau,S}\|_\star,\|T'_{\tau,S}\|_\star\le D_{\tau,S} (\tau>t)$,
we have
\begin{align*}
&\left(\sum_{S\in\mathcal{S}_t}\|H_{t,S,i}(\{T_{\tau,S'}\}_{\tau>t})-H_{t,S,i}(\{T'_{\tau,S'}\}_{\tau>t})\|_F^2\right)^{1/2}\\
&\le
L_{H,t}(D_{>t})\left(\sum_{\tau=t+1}^L\sum_{S\in\mathcal{S}_\tau}\|T_{\tau,S}H_{\tau,S,i}(\{T_{\nu,S'}\}_{\nu>\tau})-T'_{\tau,S}H_{\tau,S,i}(\{T'_{\nu,S'}\}_{\nu>\tau})\|_F^2\right)^{1/2}.
\end{align*}
\end{assumption}

\begin{assumption}[Regular frozen suffix on reachable activation sets for $L$ attention modules]\label{ast:L-attn-suffix}
Fix trace-norm radii $(D_{t,S})$.
Assume there exists a finite constant $L_{L\mathrm{attn}}(D)$ such that for all $i\in[N]$ and all pairs of reachable collections $\{Z_{t,S,i}\}_{t,S},\{Z'_{t,S,i}\}_{t,S}$,
one has
\[
\left\|\psi_{L\mathrm{attn},i}(\{Z_{t,S,i}\}_{t,S})-\psi_{L\mathrm{attn},i}(\{Z'_{t,S,i}\}_{t,S})\right\|_2\le L_{L\mathrm{attn}}(D)\left(\sum_{t=1}^L\sum_{S\in\mathcal{S}_t}\|Z_{t,S,i}-Z'_{t,S,i}\|_F^2\right)^{1/2}.
\]
\end{assumption}

For the single-block case, let $\sigma_i\in\{-1,1\}^m$ be an $m$-dimensional vector with independent Rademacher entries.
Let $E\in\{-1,1\}^{m\times N}$ be the matrix whose $i$-th column is $\sigma_i$.
Then $\sum_{i=1}^N\sigma_ih_i^\top=EH^\top$.

We also define
\[
L_\psi(D,R_h)\coloneqq\max_{i\in[N]}\sup_{z\in\mathcal{Z}_{\mathrm{act}}(D,R_h)}\|D\psi_i(z)\|_2.
\]

\begin{assumption}[Regular single-block frozen suffix on reachable activation sets]\label{ast:single-suffix-lipschitz}
Fix $D\ge0$.
Assume that for each $i\in[N]$, the map $\psi_i$ is $C^1$ on an open neighborhood of $\mathcal{Z}_{\mathrm{act}}(D,R_h)$.
\end{assumption}

The set $\mathcal{Z}_{\mathrm{act}}(D,R_h)$ is closed and bounded, hence compact.
Since each $D\psi_i$ is continuous, the supremum $L_\psi(D,R_h)$ is finite.

The next two lemmas formalize the two “frozen suffix” decompositions used in the analysis. The first applies to a single adapted module, and the second extends the idea to $L$ consecutive adapted attention modules. 

\begin{lemma}[Single-cut suffix representation for one adapted module]\label{lem:frozen-suffix-multi}
Assume Assumption~\ref{ast:frozen-suffix-new} holds.
Let $\mathcal{S}$ index adapted linear maps that lie behind one computation-graph cut, in the following sense:
for every training sample $i\in[N]$, the input tensor $H_{S,i}$ to each adapted map $S$ is computed before any map in $\mathcal{S}$ is applied and is independent of the whole update tuple $\{T_{S'}\}$.
Then there exist deterministic maps
\[
\psi_i:\prod_{S\in\mathcal{S}}\R^{m_S\times\varrho_{S,i}}\to\R^K
\]
depending only on $x_i$ and the frozen weights such that, for all updates $\{T_S\}$,
\[
f_{\{W_{S,0}+T_S\}}(x_i)=\psi_i\Bigl(\{(W_{S,0}+T_S)H_{S,i}\}\Bigr)=\psi_i\Bigl(\{W_{S,0}H_{S,i}+T_SH_{S,i}\}\Bigr).
\]
Consequently, for one adapted module, the empirical loss depends on each block $T_S$ only through the bottleneck outputs $\{T_SH_{S,i}\}_{i=1}^N$, or equivalently through the stacked linear operator $\Phi_S(T_S)=T_SH_S$.
\end{lemma}

\begin{proof}[Proof of Lemma~\ref{lem:frozen-suffix-multi}]
Fix $i\in[N]$.
By the single-cut assumption, the tensors $\{H_{S,i}\}$ and all cached quantities produced before the adapted module are fixed functions of $x_i$ and the frozen weights.
Cut the computation graph at the outputs $Z_{S,i}\coloneqq(W_{S,0}+T_S)H_{S,i},S\in\mathcal{S}$.
After this cut, all remaining operations, residual connections, attention operations, nonlinearities, and output layers are frozen and deterministic.
Therefore, the downstream computation defines a deterministic suffix map
\[
\psi_i:\prod_{S\in\mathcal{S}}\R^{m_S\times\varrho_{S,i}}\to\R^K
\]
from the collection $\{Z_{S,i}\}$ to the logits.
Substituting $Z_{S,i}=W_{S,0}H_{S,i}+T_SH_{S,i}$ gives the displayed representation.
Thus, changing $T_S$ while keeping the rest of the module fixed can affect the empirical loss only through the finite collection $\{T_SH_{S,i}\}_{i=1}^N$.
\end{proof}

\begin{lemma}[Multi-cut decomposition for $L$ consecutive adapted attention modules]\label{lem:frozen-suffix-L-attn}
Assume Assumption~\ref{ast:frozen-suffix-new} holds.
Consider $L$ consecutive attention modules, indexed so that larger $t$ is more upstream.
For each stage $t$, the token matrix $H_{t,S,i}(\{T_{\tau,S'}\}_{\tau>t})$ entering block $S$ is allowed to depend on upstream updates, but not on the current block update $T_{t,S}$.
Then there exist deterministic maps
\[
\psi_{L\mathrm{attn},i}:\prod_{t=1}^L\prod_{S\in\mathcal{S}_t}\R^{m_{t,S}\times\varrho_{t,S,i}}\to\R^K
\]
such that, for all update collections $\{T_{t,S}\}$,
\[
f_{\{W_{t,S,0}+T_{t,S}\}}(x_i)=\psi_{L\mathrm{attn},i}\Bigl(\{(W_{t,S,0}+T_{t,S})H_{t,S,i}(\{T_{\tau,S'}\}_{\tau>t})\}_{t\in[L],S\in\mathcal{S}_t}\Bigr).
\]
\end{lemma}

\begin{proof}[Proof of Lemma~\ref{lem:frozen-suffix-L-attn}]
Fix $i\in[N]$.
We cut the computation graph once at each adapted attention module.
At the most upstream adapted stage $L$, the inputs to the adapted projections are fixed by the frozen prefix, so the stage outputs have the form
\[
Z_{L,S,i}=(W_{L,S,0}+T_{L,S})H_{L,S,i}.
\]
Because the computation between stage $L$ and stage $L-1$ is deterministic, the token matrices entering stage $L-1$ are deterministic functions of the stage-$L$ cut outputs.
Thus, the stage-$L-1$ adapted outputs have the form
\[
Z_{L-1,S,i}=(W_{L-1,S,0}+T_{L-1,S})H_{L-1,S,i}(\{T_{\tau,S'}\}_{\tau>L-1}).
\]
Repeating this argument down to stage $1$ shows that every adapted output can be written as
\[
Z_{t,S,i}=(W_{t,S,0}+T_{t,S})H_{t,S,i}(\{T_{\tau,S'}\}_{\tau>t}).
\]
After all these cuts, the remaining suffix from the collection $\{Z_{t,S,i}\}_{t,S}$ to the logits is frozen and deterministic.
This suffix is the map $\psi_{L\mathrm{attn},i}$.
The displayed formula follows by substituting the expressions for the cut outputs.
\end{proof}

The Rademacher-complexity proof below requires applying vector contraction to functions whose natural domains are not all of Euclidean space, but only the reachable activation sets induced by trace-norm-bounded updates. To avoid introducing any extension assumption, we use the classical real-valued Lipschitz extension formula of \citet{mcshane1934extension}. The lemma is stated in the metric-space form needed later.

\begin{lemma}[McShane extension theorem]\label{lem:mcshane-extension}
Let $(X,d)$ be a metric space, let $\emptyset\neq A\subseteq X$, and let $f:A\to\R$ be $L$-Lipschitz.
Define
\[
\bar{f}(x)\coloneqq\inf_{a\in A}\bigl(f(a)+Ld(x,a)\bigr),x\in X.
\]
Then $\bar{f}:X\to\R$ is an $L$-Lipschitz extension of $f$.
\end{lemma}

\begin{proof}[Proof of Lemma~\ref{lem:mcshane-extension}]
Fix $x\in A$.
For every $a\in A$, the $L$-Lipschitz property of $f$ gives
\[
f(a)+Ld(x,a)\ge f(x).
\]
Taking $a=x$ shows equality is attained, hence $\bar{f}(x)=f(x)$.
Thus $\bar{f}$ extends $f$.

Now fix $x,y\in X$.
For every $a\in A$,
\[
\bar{f}(x)\le f(a)+Ld(x,a)\le f(a)+Ld(y,a)+Ld(x,y).
\]
Taking the infimum over $a\in A$ yields $\bar{f}(x)\le\bar{f}(y)+Ld(x,y)$.
Exchanging the roles of $x$ and $y$ gives $\bar{f}(y)\le\bar{f}(x)+Ld(x,y)$.
Therefore $|\bar{f}(x)-\bar{f}(y)|\le Ld(x,y)$, so $\bar{f}$ is $L$-Lipschitz.
\end{proof}

We next show the standard Rademacher deviation bound that converts a complexity estimate into a high-probability generalization guarantee.

\begin{lemma}[Uniform deviation from conditional Rademacher complexity]\label{lem:rad-to-uniform}
Let $\mathcal{F}$ be a class of functions $f:\mathcal{X}\times\mathcal{Y}\to[0,B]$.
Let $S_{\mathrm{data}}=\{(x_i,y_i)\}_{i=1}^N$ be i.i.d.\ from $\mathcal{P}$ and write
\[
\widehat{\E}_Sf\coloneqq\frac1N\sum_{i=1}^Nf(x_i,y_i),\qquad \E f\coloneqq\E_{(x,y)\sim\mathcal{P}}[f(x,y)].
\]
Define the (conditional) empirical Rademacher complexity
\[
\mathfrak{R}_{S_{\mathrm{data}}}(\mathcal{F})\coloneqq\E_\varepsilon\left[\sup_{f\in\mathcal{F}}\frac1N\sum_{i=1}^N\varepsilon_if(x_i,y_i)\mid S_{\mathrm{data}}\right],
\]
where $\varepsilon_1,\dots,\varepsilon_N$ are i.i.d.\ uniform signs.
Then for every $\eta\in(0,1)$, with probability at least $1-\eta$ over $S\sim\mathcal{P}^N$,
\[
\sup_{f\in\mathcal{F}}\Bigl(\E f-\widehat{\E}_Sf\Bigr)\le2\mathfrak{R}_{S_{\mathrm{data}}}(\mathcal{F})+3B\sqrt{\frac{\log(2/\eta)}{2N}}.
\]
\end{lemma}

\begin{proof}[Proof of Lemma~\ref{lem:rad-to-uniform}]
Define
\[
g(S)\coloneqq\sup_{f\in\mathcal{F}}\bigl(\E f-\widehat{\E}_Sf\bigr).
\]
If we replace $(x_i,y_i)$ by an independent copy $(x_i',y_i')$ while keeping all other samples fixed, then for any fixed $f$,
\[
\left|\widehat{\E}_Sf-\widehat{\E}_{S_{\mathrm{data}}^{(i)}} f\right|=\frac1N\left|f(x_i,y_i)-f(x_i',y_i')\right|\le\frac{B}{N}.
\]
Taking the supremum over $f$ implies $|g(S)-g(S_{\mathrm{data}}^{(i)})|\le B/N$.
McDiarmid's inequality therefore yields: with probability at least $1-\eta/2$,
\[
g(S)\le\E[g(S)]+B\sqrt{\frac{\log(2/\eta)}{2N}}.
\]
Let $S'=\{(x_i',y_i')\}_{i=1}^N$ be an i.i.d.\ ghost sample from $\mathcal{P}^N$, independent of $S$.
Then
\[
\E[g(S)]=\E_S\left[\sup_{f\in\mathcal{F}}\Bigl(\E_{S'}\widehat{\E}_{S'} f-\widehat{\E}_Sf\Bigr)\right]\le\E_{S,S'}\left[\sup_{f\in\mathcal{F}}\Bigl(\widehat{\E}_{S'} f-\widehat{\E}_Sf\Bigr)\right].
\]
Introduce i.i.d.\ signs $\varepsilon_1,\dots,\varepsilon_N$, independent of $(S,S')$, and use that $(S,S')$ and $(S',S)$ have the same distribution to obtain
\[
\E_{S,S'}\left[\sup_{f}\frac1N\sum_{i=1}^N\bigl(f(x_i',y_i')-f(x_i,y_i)\bigr)\right]=\E_{S,S',\varepsilon}\left[\sup_{f}\frac1N\sum_{i=1}^N\varepsilon_i\bigl(f(x_i',y_i')-f(x_i,y_i)\bigr)\right].
\]
By the triangle inequality and symmetry,
\[
\E[g(S)]\le2\E_S\left[\E_\varepsilon\left[\sup_{f\in\mathcal{F}}\frac1N\sum_{i=1}^N\varepsilon_if(x_i,y_i)\mid S_{\mathrm{data}}\right]\right]=2\E_S[\mathfrak{R}_{S_{\mathrm{data}}}(\mathcal{F})].
\]
Fix $\varepsilon$ and define
\[
h_\varepsilon(S)\coloneqq\sup_{f\in\mathcal{F}}\frac1N\sum_{i=1}^N\varepsilon_if(x_i,y_i).
\]
Replacing one sample changes $h_\varepsilon(S)$ by at most $B/N$, hence $\mathfrak{R}_{S_{\mathrm{data}}}(\mathcal{F})=\E_\varepsilon[h_\varepsilon(S)|S]$ also satisfies bounded differences with constant $B/N$.
McDiarmid implies: with probability at least $1-\eta/2$,
\[
\E_S[\mathfrak{R}_{S_{\mathrm{data}}}(\mathcal{F})]\le\mathfrak{R}_{S_{\mathrm{data}}}(\mathcal{F})+B\sqrt{\frac{\log(2/\eta)}{2N}}.
\]
On the intersection of the two $1-\eta/2$ events and using $\E[g(S)]\le2\E[\mathfrak{R}_{S_{\mathrm{data}}}(\mathcal{F})]$, we get
\[
g(S)\le2\mathfrak{R}_{S_{\mathrm{data}}}(\mathcal{F})+3B\sqrt{\frac{\log(2/\eta)}{2N}}.
\]
as desired. 
\end{proof}
Combining these results gives the main single-block generalization result. The bound reduces uniform generalization over the trace-norm ball to the empirical Rademacher complexity of the induced loss class, and then to the spectral norm of a signed feature matrix.
\begin{theorem}[Uniform generalization for trace-norm-bounded updates in frozen networks]\label{thm:gen-full-single}
Assume Assumptions~\ref{ast:frozen-suffix-new},\ref{ast:bounded-lipschitz-loss},\ref{ast:bounded-activations}, and \ref{ast:single-suffix-lipschitz} hold.
Then with probability at least $1-\eta$ over $S_{\mathrm{data}}\sim\mathcal{P}^N$,
\[
\sup_{T\in\B_\star(D)}\Bigl(L(T)-\mathcal{L}^{\mathrm{vanilla}}_{S_{\mathrm{data}}}(T)\Bigr)\le2\mathfrak{R}_{S_{\mathrm{data}}}(\mathcal{F}_D)+3B\sqrt{\frac{\log(2/\eta)}{2N}}.
\]
Moreover,
\begin{equation}\label{eqn:rad-main}
\mathfrak{R}_{S_{\mathrm{data}}}(\mathcal{F}_D)\le\frac{\sqrt{2}L_\ell L_\psi(D,R_h)D}{N}\E_\sigma\bigl[\|EH^\top\|_2\mid S_{\mathrm{data}}\bigr],
\end{equation}
\end{theorem}
Theorem~\ref{thm:gen-full-single} reduces uniform generalization over the trace-norm ball $\B_\star(D)$ to the spectral norm of a random signed feature matrix $EH^\top$.

\begin{proof}[Proof of Theorem~\ref{thm:gen-full-single}]
The first display follows immediately from Lemma~\ref{lem:rad-to-uniform} applied to the bounded class $\mathcal{F}_D$.
It remains to prove \eqref{eqn:rad-main}.
For each $i\in[N]$, define
\[
K_i(D)\coloneqq\{Th_i:\|T\|_\star\le D\}\subseteq\R^m,\phi_i(z)\coloneqq\ell\bigl(\psi_i(W_0h_i+z),y_i\bigr),z\in K_i(D).
\]
Whenever $z=Th_i$ and $z'=T'h_i$ with $\|T\|_\star,\|T'\|_\star\le D,|\phi_i(z)-\phi_i(z')|\le L_\ell L_\psi(D,R_h)\|z-z'\|_2$.
Define $\widetilde{\phi}_i(z)\coloneqq\phi_i(z)-\phi_i(0)$ on $K_i(D)$.
Then $\widetilde{\phi}_i(0)=0$, and $\widetilde{\phi}_i$ is $L_\ell L_\psi(D,R_h)$-Lipschitz on $K_i(D)$.

Moreover, for every realization of $\varepsilon$,
\[
\sup_{\|T\|_\star\le D}\frac1N\sum_{i=1}^N\varepsilon_i\phi_i(Th_i)=\sup_{\|T\|_\star\le D}\frac1N\sum_{i=1}^N\varepsilon_i\widetilde{\phi}_i(Th_i)+\frac1N\sum_{i=1}^N\varepsilon_i\phi_i(0).
\]
Taking the conditional expectation and using
\[
\E_\varepsilon\left[\sum_{i=1}^N\varepsilon_i\phi_i(0)\mid S_{\mathrm{data}}\right]=0
\]
gives
\[
\mathfrak{R}_{S_{\mathrm{data}}}(\mathcal{F}_D)=\E_\varepsilon\left[\sup_{\|T\|_\star\le D}\frac1N\sum_{i=1}^N\varepsilon_i\widetilde{\phi}_i(Th_i)\mid S_{\mathrm{data}}\right].
\]

By Lemma~\ref{lem:mcshane-extension}, each $\widetilde{\phi}_i$ extends from $K_i(D)$ to a function $\bar\phi_i:\R^m\to\R$ with the same Lipschitz constant $L_\ell L_\psi(D,R_h)$.
Since $Th_i\in K_i(D)$ for every feasible $T$, we may write
\[
\mathfrak{R}_{S_{\mathrm{data}}}(\mathcal{F}_D)=\E_\varepsilon\left[\sup_{\|T\|_\star\le D}\frac1N\sum_{i=1}^N\varepsilon_i\bar\phi_i(Th_i)\mid S_{\mathrm{data}}\right].
\]
Applying Maurer's vector contraction inequality \cite{maurer2016vector} gives
\[
\mathfrak{R}_{S_{\mathrm{data}}}(\mathcal{F}_D)\le\frac{\sqrt{2}L_\ell L_\psi(D,R_h)}{N}\E_\sigma\left[\sup_{\|T\|_\star\le D}\sum_{i=1}^N\langle\sigma_i,Th_i\rangle\mid S_{\mathrm{data}}\right].
\]

Stacking the $\sigma_i$ as columns of $E$ gives
\[
\sum_{i=1}^N\langle\sigma_i,Th_i\rangle=\left\langle\sum_{i=1}^N\sigma_ih_i^\top,T\right\rangle_F=\langle EH^\top,T\rangle_F.
\]
Therefore, by trace/spectral duality,
\[
\sup_{\|T\|_\star\le D}\sum_{i=1}^N\langle\sigma_i,Th_i\rangle=D\|EH^\top\|_2.
\]
Substituting this into the previous display yields
\eqref{eqn:rad-main}.
\end{proof}

\paragraph{Why generalization depends on $\|EH^\top\|_2$.}
The trace norm is the dual of the operator norm: $\sup_{\|T\|_\star\le D}\langle T,M\rangle=D\|M\|_2$.
After symmetrization and vector contraction, the complexity of the frozen-network loss class reduces to such a dual pairing with $M=\sum_i\sigma_ih_i^\top=EH^\top$. 
Thus, the effective capacity is governed by the spectral geometry of frozen features, not by the ambient parameter dimension $mn$.

From here, for each $(t,S)$, let $E_{t,S}$ be an independent Rademacher matrix in $\{-1,1\}^{m_{t,S}\times\sum_{i=1}^N\varrho_{t,S,i}}$.

\paragraph{Standing notation for the $L$-stage generalization argument.}
Fix a realized sample $S_{\mathrm{data}}=\{(x_i,y_i)\}_{i=1}^N$ and radii $D=(D_{t,S})_{t\in[L],S\in\mathcal{S}_t}$.
Define
\[
\Theta_D\coloneqq\{\{T_{t,S}\}_{t\in[L],S\in\mathcal{S}_t}:\|T_{t,S}\|_\star\le D_{t,S}\ \forall t\in[L],S\in\mathcal{S}_t\}.
\]
For each $(t,S)$ let $\varrho_{t,S}^{\max}\coloneqq\max_{i\in[N]}\varrho_{t,S,i}$, and define the Hilbert product space
\[
\mathcal{H}\coloneqq\bigoplus_{t=1}^L\bigoplus_{S\in\mathcal{S}_t}\R^{m_{t,S}\times\varrho_{t,S}^{\max}},
\]
equipped with the product Frobenius norm, where each $\R^{m_{t,S}\times\varrho_{t,S,i}}$ is identified with its zero-padded copy inside $\R^{m_{t,S}\times\varrho_{t,S}^{\max}}$.

For each $i\in[N]$, define the reachable observed-output set
\[
K_i(D)\coloneqq\left\{\Bigl(\{T_{t,S}H_{t,S,i}(\{T_{\tau,S'}\}_{\tau>t})\}_{S\in\mathcal{S}_t}\Bigr)_{t=1}^L:\{T_{t,S}\}\in\Theta_D\right\}\subseteq\mathcal{H}.
\]
Since the all-zero update tuple belongs to $\Theta_D$, we have $0\in K_i(D)$.

For $z\in K_i(D)$, choose any feasible tuple $\{T_{t,S}\}\in\Theta_D$ producing $z$ and define
\[
\phi_i(z)\coloneqq\ell(f_{\{W_{t,S,0}+T_{t,S}\}}(x_i),y_i),\widetilde{\phi}_i(z)\coloneqq\phi_i(z)-\phi_i(0).
\]
This is well-defined by Lemma~\ref{lem:frozen-suffix-L-attn}, because the observed-output tuple determines all full adapted outputs stage by stage and therefore determines the logits uniquely.

Finally, define
\[
\Psi_i(\{T_{t,S}\})\coloneqq\Bigl(\{T_{t,S}H_{t,S,i}(\{T_{\tau,S'}\}_{\tau>t})\}_{S\in\mathcal{S}_t}\Bigr)_{t=1}^L\in\mathcal{H}.
\]

Define the corresponding loss class
\[
\mathcal{F}_D^{(L)}\coloneqq\left\{(x,y)\mapsto\ell(f_{\{W_{t,S,0}+T_{t,S}\}}(x),y):\{T_{t,S}\}\in\Theta_D\right\}.
\]

The multi-stage attention bound has one additional difficulty compared with the single-block bound: perturbing an upstream observed output can change the token matrices seen by downstream adapted blocks.
The next deterministic lemma isolates this recursive amplification.

\begin{lemma}[Recursive triangular amplification]\label{lem:L-attn-gamma}
Fix nonnegative numbers $\rho_1,\ldots,\rho_{L-1}$.
Set
\[
\Gamma_L\coloneqq1,\qquad \Gamma_t\coloneqq\left\|
\begin{bmatrix}
1&\rho_t\\
0&\Gamma_{t+1}
\end{bmatrix}
\right\|_2(t=L-1,\ldots,1).
\]
If nonnegative numbers $a_t,b_t$ satisfy
\[
b_L\le a_L,b_t\le a_t+\rho_t\left(\sum_{\tau=t+1}^La_\tau^2\right)^{1/2}(t=1,\ldots,L-1),
\]
then for every $t\in[L]$,
\[
\left(\sum_{\tau=t}^Lb_\tau^2\right)^{1/2}\le\Gamma_t\left(\sum_{\tau=t}^La_\tau^2\right)^{1/2}.
\]
\end{lemma}

\begin{proof}[Proof of Lemma~\ref{lem:L-attn-gamma}]
We argue by backward induction on $t$.
The case $t=L$ is immediate.
Assume the claim holds for $t+1$ and set $c\coloneqq\left(\sum_{\tau=t+1}^La_\tau^2\right)^{1/2}$.
Then
\[
\left(\sum_{\tau=t}^Lb_\tau^2\right)^{1/2}\le\sqrt{(a_t+\rho_t c)^2+\Gamma_{t+1}^2c^2}=\left\|
\begin{bmatrix}
1&\rho_t\\
0&\Gamma_{t+1}
\end{bmatrix}
\binom{a_t}{c}\right\|_2\le\Gamma_t\sqrt{a_t^2+c^2},
\]
which is exactly the desired bound.
\end{proof}
The previous lemma provides the recursive norm-growth factor that controls how perturbations accumulate across consecutive stages.
Combining it with the observed output Lipschitz bound yields the following estimate for the full $L$-module attention. 
\begin{lemma}[Observed-output Lipschitzness for $L$ consecutive attention modules]\label{lem:L-attn-suffix-lipschitz}
Assume Assumptions~\ref{ast:frozen-suffix-new},\ref{ast:bounded-lipschitz-loss},\ref{ast:L-attn-token-lipschitz}, and \ref{ast:L-attn-suffix} hold.
For $t=1,\ldots,L-1$, define
\[
\|W_{t,0}\|_{\mathcal{S}_t}\coloneqq\left(\sum_{S\in\mathcal{S}_t}\|W_{t,S,0}\|_2^2\right)^{1/2},\rho_t(D)\coloneqq\|W_{t,0}\|_{\mathcal{S}_t}L_{H,t}(D_{>t}),
\]
and let $\Gamma_t(D)$ be the constants from Lemma~\ref{lem:L-attn-gamma}.
Then for any two update collections in the product trace-norm ball and every $i\in[N]$,
\begin{align*}
&\left|\ell\bigl(f_{\{W_{t,S,0}+T_{t,S}\}}(x_i),y_i\bigr)-\ell\bigl(f_{\{W_{t,S,0}+T'_{t,S}\}}(x_i),y_i\bigr)\right|\\
&\le L_\ell L_{L\mathrm{attn}}(D)\Gamma_1(D)\left(\sum_{t=1}^L\sum_{S\in\mathcal{S}_t}\left\|T_{t,S}H_{t,S,i}(\{T_{\tau,S'}\}_{\tau>t})-T'_{t,S}H_{t,S,i}(\{T'_{\tau,S'}\}_{\tau>t})\right\|_F^2\right)^{1/2}.
\end{align*}
\end{lemma}

\begin{proof}[Proof of Lemma~\ref{lem:L-attn-suffix-lipschitz}]
Fix $i$.
Define
\[
a_t\coloneqq\left(\sum_{S\in\mathcal{S}_t}\left\|T_{t,S}H_{t,S,i}(\{T_{\tau,S'}\}_{\tau>t})-T'_{t,S}H_{t,S,i}(\{T'_{\tau,S'}\}_{\tau>t})\right\|_F^2\right)^{1/2}.
\]
Also, define the actual adapted outputs
\[
Z_{t,S,i}\coloneqq(W_{t,S,0}+T_{t,S})H_{t,S,i}(\{T_{\tau,S'}\}_{\tau>t}),
\]
\[
Z'_{t,S,i}\coloneqq(W_{t,S,0}+T'_{t,S})H_{t,S,i}(\{T'_{\tau,S'}\}_{\tau>t}),
\]
and
\[
b_t\coloneqq\left(\sum_{S\in\mathcal{S}_t}\|Z_{t,S,i}-Z'_{t,S,i}\|_F^2\right)^{1/2}.
\]
For the top stage $L$, the token matrices are frozen, hence $b_L\le a_L$.
For $t<L$, write
\[
\Delta H_{t,S,i}\coloneqq H_{t,S,i}(\{T_{\tau,S'}\}_{\tau>t})-H_{t,S,i}(\{T'_{\tau,S'}\}_{\tau>t}),
\]
so that
\[
Z_{t,S,i}-Z'_{t,S,i}=W_{t,S,0}\Delta H_{t,S,i}+\Bigl(T_{t,S}H_{t,S,i}(\{T_{\tau,S'}\}_{\tau>t})-T'_{t,S}H_{t,S,i}(\{T'_{\tau,S'}\}_{\tau>t})\Bigr).
\]
Therefore
\[
b_t\le a_t+\left(\sum_{S\in\mathcal{S}_t}\|W_{t,S,0}\Delta H_{t,S,i}\|_F^2\right)^{1/2}.
\]
Using $\|AB\|_F\le\|A\|_2\|B\|_F$ and Cauchy-Schwarz across $S$,
\[
b_t\le a_t+\|W_{t,0}\|_{\mathcal{S}_t}\left(\sum_{S\in\mathcal{S}_t}\|\Delta H_{t,S,i}\|_F^2\right)^{1/2}.
\]
Assumption~\ref{ast:L-attn-token-lipschitz} gives
\[
b_t\le a_t+\rho_t(D)\left(\sum_{\tau=t+1}^La_\tau^2\right)^{1/2}.
\]
Applying Lemma~\ref{lem:L-attn-gamma},
\[
\left(\sum_{t=1}^Lb_t^2\right)^{1/2}\le\Gamma_1(D)\left(\sum_{t=1}^La_t^2\right)^{1/2}.
\]
Finally, Assumption~\ref{ast:L-attn-suffix} and the $L_\ell$-Lipschitz property of $\ell(\cdot,y_i)$ imply
\[
\left|\ell(f_T(x_i),y_i)-\ell(f_{T'}(x_i),y_i)\right|\le L_\ell L_{L\mathrm{attn}}(D)\left(\sum_{t=1}^Lb_t^2\right)^{1/2},
\]
which, combined with the previous display, proves the claim.
\end{proof}
With this estimate, we can state the multi-stage generalization bound. The result is the attention analogue of the single-block theorem above, with the random signed matrices appearing stage by stage.
\begin{theorem}[Uniform generalization for trace-norm-bounded updates in $L$ consecutive attention modules]\label{thm:gen-full-L-attn}
Assume Assumptions~\ref{ast:frozen-suffix-new},\ref{ast:bounded-lipschitz-loss},\ref{ast:L-attn-token-lipschitz}, and \ref{ast:L-attn-suffix} hold.
Fix radii $(D_{t,S})_{t\in[L],S\in\mathcal{S}_t}$.
For each $(t,S)$, let $E_{t,S}$ be an independent Rademacher matrix of the same shape as the stacked stage-$t$, block-$S$ feature matrix.
Then with probability at least $1-\eta$ over $S_{\mathrm{data}}\sim\mathcal{P}^N$,
\begin{align*}
&\sup_{\|T_{t,S}\|_\star\le D_{t,S}}\Bigl(L(\{T_{t,S}\})-\mathcal{L}^{\mathrm{vanilla}}_{S_{\mathrm{data}}}(\{T_{t,S}\})\Bigr)\\
&\le\frac{2\sqrt{2}L_\ell L_{L\mathrm{attn}}(D)\Gamma_1(D)}{N}\sum_{t=1}^L\sum_{S\in\mathcal{S}_t}D_{t,S}\E\Bigg[\sup_{\substack{\|T_{\tau,S'}\|_\star \\ \le D_{\tau,S'} (\tau>t)}}\bigl\|E_{t,S}H_{t,S}(\{T_{\tau,S'}\}_{\tau>t})^\top\bigr\|_2\Big|S_{\mathrm{data}}\Bigg]\\
&+3B\sqrt{\frac{\log(2/\eta)}{2N}}.
\end{align*}
Here, the supremum over $(\tau>t)$ is interpreted as vacuous when $t=L$.
\end{theorem}

The same statement applies verbatim to multi-head attention by taking the block family at stage $t$ to be a head-indexed collection such as $\mathcal{A}_t\times\{Q,K,V\}$, or the corresponding shared-block collection in grouped-query or multi-query attention.

\begin{proof}[Proof of Theorem~\ref{thm:gen-full-L-attn}]
By Lemma~\ref{lem:rad-to-uniform}, it suffices to bound $\mathfrak{R}_{S_{\mathrm{data}}}(\mathcal{F}_D^{(L)})$.
Fix a realized sample $S_{\mathrm{data}}$.
By Lemma~\ref{lem:L-attn-suffix-lipschitz},
\[
|\widetilde{\phi}_i(z)-\widetilde{\phi}_i(z')|\le L_\ell L_{L\mathrm{attn}}(D)\Gamma_1(D)\|z-z'\|_{\mathcal{H}}\quad \forall z,z'\in K_i(D).
\]

Because $0\in K_i(D)$ and $\widetilde{\phi}_i(0)=0$, Lemma~\ref{lem:mcshane-extension} yields an extension $\bar\phi_i:\mathcal{H}\to\R$ such that $\bar\phi_i|_{K_i(D)}=\widetilde{\phi}_i$ and
\[
|\bar\phi_i(z)-\bar\phi_i(z')|\le
L_\ell L_{L\mathrm{attn}}(D)\Gamma_1(D)\|z-z'\|_{\mathcal{H}}\quad \forall z,z'\in\mathcal{H}.
\]
Therefore
\begin{align*}
\mathfrak{R}_{S_{\mathrm{data}}}(\mathcal{F}_D^{(L)})&=\E_\varepsilon\left[\sup_{\{T_{t,S}\}\in\Theta_D}\frac1N\sum_{i=1}^N\varepsilon_i\widetilde{\phi}_i\bigl(\Psi_i(\{T_{t,S}\})\bigr)\mid S_{\mathrm{data}}\right]\\
&=\E_\varepsilon\left[\sup_{\{T_{t,S}\}\in\Theta_D}\frac1N\sum_{i=1}^N\varepsilon_i\bar{\phi}_i\bigl(\Psi_i(\{T_{t,S}\})\bigr)\mid S_{\mathrm{data}}\right].
\end{align*}

Applying Maurer's vector contraction inequality \cite{maurer2016vector} to the class $\left\{i\mapsto\Psi_i(\{T_{t,S}\}):\{T_{t,S}\}\in\Theta_D\right\}$ gives
\[
\mathfrak{R}_{S_{\mathrm{data}}}(\mathcal{F}_D^{(L)})\le C\E\Bigg[\sup_{\{T_{t,S}\}\in\Theta_D}\sum_{t=1}^L\sum_{S\in\mathcal{S}_t}\sum_{i=1}^N\langle\Sigma_{t,S,i},T_{t,S}H_{t,S,i}(\{T_{\tau,S'}\}_{\tau>t})\rangle_F\mid S_{\mathrm{data}}\Bigg],
\]
where the $\Sigma_{t,S,i}$ are independent Rademacher sign matrices of matching size and $C \coloneqq \dfrac{\sqrt{2}L_\ell L_{L\mathrm{attn}}(D)\Gamma_1(D)}{N}$.

Stacking the $\Sigma_{t,S,i}$ across $i$ gives
\[
\sum_{i=1}^N\langle\Sigma_{t,S,i},T_{t,S}H_{t,S,i}(\{T_{\tau,S'}\}_{\tau>t})\rangle_F=\langle E_{t,S}H_{t,S}(\{T_{\tau,S'}\}_{\tau>t})^\top,T_{t,S}\rangle_F.
\]
Therefore, using $\sup\limits_x \sum_j a_j(x)\le\sum_j \sup\limits_x a_j(x)$ and trace/spectral duality,
\begin{align*}
\mathfrak{R}_{S_{\mathrm{data}}}(\mathcal{F}_D^{(L)})&\le C\sum_{t=1}^L\sum_{S\in\mathcal{S}_t}\E\Bigg[\sup_{\substack{
\|T_{t,S}\|_\star\le D_{t,S}\\
\|T_{\tau,S'}\|_\star\le D_{\tau,S'}\ (\tau>t)
}}
\langle E_{t,S}H_{t,S}(\{T_{\tau,S'}\}_{\tau>t})^\top,T_{t,S}\rangle_F\Bigg]\\
&=C\sum_{t=1}^L\sum_{S\in\mathcal{S}_t}D_{t,S}\E\Bigg[\sup_{\|T_{\tau,S'}\|_\star\le D_{\tau,S'}(\tau>t)}\bigl\|E_{t,S}H_{t,S}(\{T_{\tau,S'}\}_{\tau>t})^\top\bigr\|_2\mid S_{\mathrm{data}}\Bigg].
\end{align*}
Applying Lemma~\ref{lem:rad-to-uniform} to the bounded class $\mathcal{F}_D^{(L)}$ yields the stated uniform generalization bound.
\end{proof}

\section{Proofs for Main results}

We begin with two reductions that are used repeatedly.
The first reduction says that if a loss is invariant under projecting an update onto fixed left and right subspaces, then nuclear-norm regularization never prefers the unprojected representative.
The second reduction rewrites the nuclear norm via a positive-semidefinite lift, which we can apply the Barvinok--Pataki rank bound \citep{pataki1998rank,barvinok2001remark}.

\subsection{Proofs for Section~\ref{sec:single-block}}\label{sec:proof-single-block}

\begin{lemma}[Projector compression for nuclear-norm regularization]\label{lem:proj-compression}
Let $P_{\mathrm{L}}\in\R^{m\times m}$ and $P_{\mathrm{R}}\in\R^{n\times n}$ be orthogonal projectors.
Let $\Phi:\R^{m\times n}\to\mathcal{E}$ be any map such that $\Phi(T)=\Phi(P_{\mathrm{L}}TP_{\mathrm{R}})\ \forall T\in\R^{m\times n}$.
Then $\|P_{\mathrm{L}}TP_{\mathrm{R}}\|_\star\le\|T\|_\star\ \forall T\in\R^{m\times n}$.
Consequently, if
\[
T\mapsto g(\Phi(T)+s)+\lambda\|T\|_\star,\lambda>0,
\]
admits a global minimizer, then for every global minimizer $\tilde{T}$, the compressed matrix $T^\star\coloneqq P_{\mathrm{L}}\tilde{T}P_{\mathrm{R}}$ is also a global minimizer and satisfies 
\[
\rank(T^\star)\le\min\{\rank(P_{\mathrm{L}}),\rank(P_{\mathrm{R}})\}.
\]
\end{lemma}

\begin{proof}[Proof of Lemma~\ref{lem:proj-compression}]
Since $P_{\mathrm{L}}$ and $P_{\mathrm{R}}$ are orthogonal projectors, we have $\|P_{\mathrm{L}}\|_2\le1,\|P_{\mathrm{R}}\|_2\le1$.
By Schatten-norm submultiplicativity,
\[
\|P_{\mathrm{L}}TP_{\mathrm{R}}\|_\star\le\|P_{\mathrm{L}}\|_2\|T\|_\star\|P_{\mathrm{R}}\|_2\le\|T\|_\star.
\]

Let $\widetilde{T}$ be any global minimizer and define $T^\star\coloneqq P_{\mathrm{L}}\widetilde{T}P_{\mathrm{R}}$.
By hypothesis, $\Phi(T^\star)=\Phi(\widetilde{T})$, so the loss term is unchanged.
The nuclear norm does not increase, so
\[
g(\Phi(T^\star)+s)+\lambda\|T^\star\|_\star\le g(\Phi(\widetilde{T})+s)+\lambda\|\widetilde{T}\|_\star.
\]
Since $\widetilde{T}$ is globally optimal, equality must hold, and $T^\star$ is also globally optimal.
We also have
\[
\rank(T^\star)\le\min\{\rank(P_{\mathrm{L}}),\rank(P_{\mathrm{R}})\}.
\]
\end{proof}

\begin{lemma}[Nuclear-norm lifting: equivalence of infima]\label{lem:lifting-equivalence}
Let $g:\mathcal{E}\to\R$ be any function and fix $s\in\mathcal{E}$.
Then
\[
\min_{T\in\R^{m\times n}}F_{\mathrm{primal}}(T)=\min_{Z\succeq0}F_{\mathrm{lift}}(Z).
\]
Moreover:
\begin{enumerate}
\item For every $T$ there exists $Z\succeq0$ with $\bar{Z}=T$ and $F_{\mathrm{lift}}(Z)=F_{\mathrm{primal}}(T)$.
\item For every $Z\succeq0$, $F_{\mathrm{primal}}(\bar{Z})\le F_{\mathrm{lift}}(Z)$.
\item If $Z^\star$ minimizes $F_{\mathrm{lift}}$, then $T^\star=\bar{Z}^\star$ minimizes $F_{\mathrm{primal}}$.
\end{enumerate}
\end{lemma}

\begin{proof}[Proof of Lemma~\ref{lem:lifting-equivalence}]
Fix $Z\succeq0$ and factor $Z=RR^\top$ with 
\(
R=\begin{bmatrix}U\\V\end{bmatrix}.
\)
Then $\bar{Z}=UV^\top$ and $\|UV^\top\|_\star\le\frac12(\|U\|_F^2+\|V\|_F^2)=\frac12\tr(Z)$, so $F_{\mathrm{primal}}(\bar{Z})\le F_{\mathrm{lift}}(Z)$.

Conversely, fix $T\in\R^{m\times n}$ and let $T=U_T\Sigma V_T^\top$ be a compact SVD.
Set $U\coloneqq U_T\Sigma^{1/2}\in\R^{m\times r}$ and $V\coloneqq V_T\Sigma^{1/2}\in\R^{n\times r}$, where $r=\rank(T)$.
Then $UV^\top=T$ and $\|U\|_F^2=\|V\|_F^2=\tr(\Sigma)=\|T\|_\star$.
Define
\[
Z\coloneqq
\begin{bmatrix}
UU^\top&UV^\top\\
VU^\top&VV^\top
\end{bmatrix}=\begin{bmatrix}U\\V\end{bmatrix}\begin{bmatrix}U\\V\end{bmatrix}^\top\succeq0,
\]
so $\bar{Z}=UV^\top=T$ and $\tr(Z)=\|U\|_F^2+\|V\|_F^2=2\|T\|_\star$.
Hence $F_{\mathrm{lift}}(Z)=F_{\mathrm{primal}}(T)$.
Taking infima yields equality, and the minimizer mapping follows.
\end{proof}

\begin{proof}[Proof of Theorem~\ref{thm:rank-single}]
Define
\[
F(T)\coloneqq g(\Phi(T)+s)+\lambda\|T\|_\star .
\]
Since $\lambda>0$, the map $T\mapsto \lambda\|T\|_\star$ is coercive.
Since $g$ is bounded below and lower semicontinuous, and since $\Phi$ is continuous, the objective $F$ is bounded below, lower semicontinuous, and coercive on the finite-dimensional space $\R^{m\times n}$.
Therefore, $F$ attains a global minimum.

Fix a global minimizer $\widetilde{T}$ and let $\widetilde{T}=AB^\top$ be a balanced factorization obtained from a compact singular value decomposition, so that
\[
\frac{1}{2}\left(\|A\|_F^2+\|B\|_F^2\right)=\|\widetilde{T}\|_\star.
\]
Define
\[
\widetilde{Z}\coloneqq\begin{bmatrix}A\\B\end{bmatrix}\begin{bmatrix}A\\B\end{bmatrix}^{\top}.
\]
Then $\widetilde{Z}\succeq0$, $\overline{\widetilde{Z}}=\widetilde{T}$, and $\operatorname{tr}(\widetilde{Z})=2\|\widetilde{T}\|_\star$.

Consider the conditional trace-minimization problem
\[
\min_{Z\succeq0}\operatorname{tr}(Z)\text{ subject to }\Phi(\bar{Z})=\Phi(\widetilde{T}) .
\]
This feasible set is nonempty because it contains $\widetilde{Z}$.
The objective is bounded below by $0$.
Moreover, every trace sublevel set is compact, because $Z\succeq0$ implies $\|Z\|_F\le\operatorname{tr}(Z)$.
Hence, the conditional trace-minimization problem attains its minimum.
Let $v^\star$ denote its minimum value.

The set $\mathcal{C}^\star\coloneqq\left\{Z\succeq0:\Phi(\bar{Z})=\Phi(\widetilde{T}),\operatorname{tr}(Z)=v^\star\right\}$
is nonempty and compact.
Choose an extreme point $Z^\star$ of $\mathcal{C}^\star$.
Since $\mathcal{C}^\star$ is the set of minimizers of the linear functional $Z\mapsto\operatorname{tr}(Z)$ over the spectrahedron $\mathcal{C}\coloneqq\{Z\succeq0:\Phi(\bar{Z})=\Phi(\widetilde{T})\}$, the set $\mathcal{C}^\star$ is a face of $\mathcal{C}$.
Therefore, $Z^\star$ is also an extreme point of $\mathcal{C}$.

Let $e_1,\ldots,e_{d_{\Phi}}$ be an orthonormal basis of $\mathcal{R}(\Phi)$.
For $\ell=1,\ldots,d_{\Phi}$, define
\[
A_\ell\coloneqq\frac{1}{2}
\begin{bmatrix}
\mathbf{0}&\Phi^\star(e_\ell)\\
\Phi^\star(e_\ell)^\top&\mathbf{0}
\end{bmatrix}.
\]
Then the constraint $\Phi(\bar{Z})=\Phi(\widetilde{T})$ is equivalent to the independent affine constraints
\[
\langle A_\ell,Z\rangle=\langle e_\ell,\Phi(\widetilde{T})\rangle_{\mathcal{E}},\ell=1,\ldots,d_{\Phi}.
\]
The independence follows because $\Phi^\star$ is injective on $\mathcal{R}(\Phi)$.

By the Barvinok-Pataki rank bound for extreme points of spectrahedra with $d_{\Phi}$ independent affine constraints \citep{pataki1998rank,barvinok2001remark}, $Z^\star$ satisfies
\[
\frac{\rank(Z^\star)\bigl(\rank(Z^\star)+1\bigr)}{2}\le d_{\Phi}.
\]
Set $T^\star\coloneqq\overline{Z^\star}$.
Then $\Phi(T^\star)=\Phi(\overline{Z^\star})=\Phi(\widetilde{T})$.
Furthermore, Lemma~\ref{lem:lifting-equivalence} gives
\[
F(T^\star)\le g(\Phi(\widetilde{T})+s)+\frac{\lambda}{2}\operatorname{tr}(Z^\star)\le g(\Phi(\widetilde{T})+s)+\frac{\lambda}{2}\operatorname{tr}(\widetilde{Z})=F(\widetilde{T}).
\]
Since $\widetilde{T}$ is globally optimal, $T^\star$ is also globally optimal.
Finally, $\rank(T^\star)\le\rank(Z^\star)$, so
\[
\frac{\rank(T^\star)\bigl(\rank(T^\star)+1\bigr)}{2}\le d_{\Phi}.
\]
Solving this quadratic inequality gives the displayed rank bound.
\end{proof}

\begin{proof}[Proof of Theorem~\ref{thm:rank-L-stage}]
Let
\[
F(\{T_{t,S}\})\coloneqq\mathcal{L}(\{T_{t,S}\})+\sum_{t=1}^L\sum_{S\in\mathcal{S}_t}\lambda_{t,S}\|T_{t,S}\|_\star.
\]
Because all $\lambda_{t,S}>0$, the regularizer is coercive on the finite-dimensional product space $\prod_{t,S}\R^{m_{t,S}\times n_{t,S}}$.
By Assumption~\ref{ast:meas-wp-new-new}, each stage map in the recursion of Definition~\ref{def:L-stage-block-lin-meas} is continuous, hence the map from the parameter tuple to $z_1$ is continuous; since $g_0$ is bounded below and lower semicontinuous, $\mathcal{L}=g_0(z_1)$ is bounded below and lower semicontinuous.
Therefore, $F$ attains a global minimum.

Fix any global minimizer of $F$.
We now process the stages from upstream to downstream, i.e.\ in the order $t=L,L-1,\ldots,1$.

Assume that all stages $u>t$ have already been fixed once and for all.
Then the corresponding state $z_{t+1}^\star$ is fixed and will not change in the subsequent steps.
Now fix a block $S\in\mathcal{S}_t$ and hold every block other than $T_{t,S}$ fixed.
The resulting conditional objective has the form
\[
T\mapsto\tilde{g}_{t,S}\bigl(\Phi_{t,S}(T,z_{t+1}^\star)\bigr)+\lambda_{t,S}\|T\|_\star,
\]
where $\tilde{g}_{t,S}$ is obtained by fixing all companion inputs and composing the remaining stage maps with $g_0$.
Because the companion blocks and $z_{t+1}^\star$ are fixed, all maps below stage $t$ are continuous functions of the single argument, and $g_0$ is lower semicontinuous and bounded below.
Thus $\tilde{g}_{t,S}$ is bounded below and lower semicontinuous on the relevant finite-dimensional space.

Since the current value of $T_{t,S}$ is already globally optimal for this conditional problem with all companion blocks fixed, Theorem~\ref{thm:rank-single} allows us to replace it by another conditional minimizer $T_{t,S}^\star$ with the same $\Phi_{t,S}$ satisfying
\[
\frac{r_{t,S}(r_{t,S}+1)}{2}\le d_{t,S}(z_{t+1}^\star),r_{t,S}\coloneqq\rank(T_{t,S}^\star).
\]
The equality of the resulting values implies that the stage-$t$ state $z_t$ and all downstream states $z_{t-1},\ldots,z_1$ are unchanged by the replacement.
Thus, the global objective value is preserved, and all blocks already processed remain valid.

Proceeding inductively through $t=L,L-1,\ldots,1$ yields a global minimizer $\{T_{t,S}^\star\}$ such that for every $t\in[L]$ and $S\in\mathcal{S}_t$,
\[
\frac{r_{t,S}(r_{t,S}+1)}{2}\le d_{t,S}(z_{t+1}^\star)\le d_{t,S}^{\max}.
\]
The final scalar bound follows from the implication
\[
\frac{r(r+1)}{2}\le d\Longrightarrow r\le\left\lfloor\frac{\sqrt{8d+1}-1}{2}\right\rfloor.
\]
\end{proof}

\subsection{Proofs for Section~\ref{sec:attn-bounds}}\label{sec:proof-attn}

The proof of Theorem~\ref{thm:rank-attn} has three ingredients: row-wise score centering for softmax, a blockwise projector-compression principle at a rank-minimal global minimizer, and explicit identification of the query, key, and value subspaces that remain observable after the frozen attention computation.

\begin{lemma}[Row-wise softmax depends only on the centered score matrix]\label{lem:softmax-row-center}
Let $L'\in\N$ and let $P_{\mathbf{1}_{L'}}^\perp=I_{L'}-\frac{1}{L'}\mathbf{1}_{L'}\mathbf{1}_{L'}^\top$.
Let $C\in(\R\cup\{-\infty\})^{a\times L'}$ be such that each row of $C$ has at least one finite entry.
Then for every $M\in\R^{a\times L'}$,
\[
\mathrm{softmax}_{\mathrm{row}}(M+C)=\mathrm{softmax}_{\mathrm{row}}(MP_{\mathbf{1}_{L'}}^\perp+C).
\]
In particular, when $C=0$, $\mathrm{softmax}_{\mathrm{row}}(M)=\mathrm{softmax}_{\mathrm{row}}(MP_{\mathbf{1}_{L'}}^\perp)$.
\end{lemma}

\begin{proof}[Proof of Lemma~\ref{lem:softmax-row-center}]
Fix a row index $i\in[a]$ and define $J_i\coloneqq\{j\in[L']:C_{ij}>-\infty\}$.
By assumption, $J_i\neq\emptyset$.
Set $c_i\coloneqq\frac{1}{L'}\sum_{j=1}^{L'}M_{ij}$.
Then, for every $j\in J_i$, $(MP_{\mathbf{1}_{L'}}^\perp+C)_{ij}=(M+C)_{ij}-c_i$.
Hence
\[
\frac{\exp((MP_{\mathbf{1}_{L'}}^\perp+C)_{ij})}{\sum_{k\in J_i}\exp((MP_{\mathbf{1}_{L'}}^\perp+C)_{ik})}=\frac{e^{-c_i}\exp((M+C)_{ij})}{e^{-c_i}\sum_{k\in J_i}\exp((M+C)_{ik})}=\frac{\exp((M+C)_{ij})}{\sum_{k\in J_i}\exp((M+C)_{ik})}.
\]
For $j\notin J_i$, both $(M+C)_{ij}$ and $(MP_{\mathbf{1}_{L'}}^\perp+C)_{ij}$ are equal to $-\infty$, so both rows assign zero softmax mass to coordinate $j$.
Therefore, the $i$-th row of the two row-wise softmax matrices coincide.
Since $i$ was arbitrary, the matrices are equal.
The special case $C=0$ is immediate.
\end{proof}

\begin{theorem}[Self-consistent projector compression]\label{thm:rank-full-selfconsistent}
Let $\lambda_S>0$ for every $S\in\mathcal{S}$, and define
\[
F(\{T_S\}_{S\in\mathcal{S}})\coloneqq\mathcal{L}(\{T_S\}_{S\in\mathcal{S}})+\sum_{S\in\mathcal{S}}\lambda_S\|T_S\|_\star.
\]
Assume that $F$ admits a global minimizer.
Suppose that for each block $S$ and every fixed companion tuple $\{T_{S'}\}_{S'\ne S}$ there exist orthogonal projectors $P_S^{\mathrm{L}}(\{T_{S'}\}_{S'\ne S})\in\R^{m_S\times m_S},P_S^{\mathrm{R}}(\{T_{S'}\}_{S'\ne S})\in\R^{n_S\times n_S}$,
such that
\[
\mathcal{L}(\{T_{S'}\}_{S'\ne S},T_S)=\mathcal{L}\Bigl(\{T_{S'}\}_{S'\ne S},P_S^{\mathrm{L}}(\{T_{S'}\}_{S'\ne S})T_SP_S^{\mathrm{R}}(\{T_{S'}\}_{S'\ne S})\Bigr)\ \forall T_S.
\]
Then there exists a global minimizer $\{T_S^\star\}$ minimizing
\[
\sum_{S\in\mathcal{S}}\rank(T_S)
\]
over the global minimizer set, and every such minimizer satisfies
\[
\rank(T_S^\star)\le\min\Bigl\{\rank\bigl(P_S^{\mathrm{L}}(\{T_{S'}^\star\}_{S'\ne S})\bigr),\rank\bigl(P_S^{\mathrm{R}}(\{T_{S'}^\star\}_{S'\ne S})\bigr)\Bigr\}.
\]
\end{theorem}

\begin{proof}[Proof of Theorem~\ref{thm:rank-full-selfconsistent}]
Since $\rank(T_S)\in\{0,1,\dots,\min(m_S,n_S)\}$ for each block $S$, the set of possible values of $\sum_S\rank(T_S)$ is finite, and there exists a global minimizer $\{T_S^\star\}\in\argmin{F}$ for which $\sum_S\rank(T_S)$ is minimal.

Fix a block $S\in\mathcal{S}$ and abbreviate $P_S^{\mathrm{L},\star}\coloneqq P_S^{\mathrm{L}}(\{T_{S'}^\star\}_{S'\ne S}),P_S^{\mathrm{R},\star}\coloneqq P_S^{\mathrm{R}}(\{T_{S'}^\star\}_{S'\ne S})$.
Define $\widehat{T}_S\coloneqq P_S^{\mathrm{L},\star}T_S^\star P_S^{\mathrm{R},\star}$.
By the assumed invariance of the loss,
\[
\mathcal{L}(\{T_{S'}^\star\}_{S'\ne S},\widehat{T}_S)=\mathcal{L}(\{T_{S'}^\star\}_{S'\ne S},T_S^\star).
\]
By Lemma~\ref{lem:proj-compression}, $\|\widehat{T}_S\|_\star\le\|T_S^\star\|_\star$.
Hence $F(\{T_{S'}^\star\}_{S'\ne S},\widehat{T}_S)\le F(\{T_S^\star\})$.
Since $\{T_S^\star\}$ is already globally optimal, the new tuple is also a global minimizer.

If $\rank(T_S^\star)>\min\{\rank(P_S^{\mathrm{L},\star}),\rank(P_S^{\mathrm{R},\star})\}$, then 
\[
\rank(\widehat{T}_S)\le\min\{\rank(P_S^{\mathrm{L},\star}),\rank(P_S^{\mathrm{R},\star})\}<\rank(T_S^\star),
\]
so, the new global minimizer would have strictly smaller total rank, contradicting the choice of $\{T_S^\star\}$.
Therefore, the displayed rank bound must hold for every block $S$.
\end{proof}

\begin{proof}[Proof of Theorem~\ref{thm:rank-attn}]
By Lemma~\ref{lem:frozen-suffix-L-attn}, for each sample $i$ the network output is a deterministic function of the collection
\[
\left\{(W_{t,S,0}^{(a)}+T_{t,S}^{(a)})H_{t,S,i}(\{T_{\tau,S'}^{(a')}\}_{\tau>t})\right\}_{(t,a,S)\in\widetilde{\mathcal{S}}}.
\]
Therefore the empirical loss exhibits $L$-stage blockwise linear transformation dependence in the sense of Definition~\ref{def:L-stage-block-lin-meas}, with linear mappings $\Phi_{t,(a,S)}(T,z)=TH_{t,S}(z)$.
Assumption~\ref{ast:meas-wp-new-new} lets us invoke Theorem~\ref{thm:rank-L-stage}, so the corresponding trace-norm regularized objective admits a global minimizer.
We now verify the projector-invariance hypothesis of Theorem~\ref{thm:rank-full-selfconsistent} for the same objective, viewed as a multi-block problem indexed by $\widetilde{\mathcal{S}}$.

Fix an adapted block $(t,a,S)\in\widetilde{\mathcal{S}}$ and hold all companion blocks fixed.
These companions determine a fixed upstream state $z_{t+1}$ and therefore fixed samplewise input matrices $H_{t,Q,i},H_{t,K,i},H_{t,V,i}(i\in[N])$.

First, consider the case $S=Q$.
Hold all blocks other than $T_{t,Q}^{(a)}$ fixed.
Define
\[
\widetilde{K}_{t,i}^{(a)}\coloneqq(W_{t,K,0}^{(a)}+T_{t,K}^{(a)})H_{t,K,i},\quad \widetilde{V}_{t,i}^{(a)}\coloneqq(W_{t,V,0}^{(a)}+T_{t,V}^{(a)})H_{t,V,i},
\] 
and for $T\in\R^{d_{k,t}\times n_{t,Q}}$ set
\[
A_{t,Q,i}^{(a)}(T)\coloneqq\mathrm{softmax}_{\mathrm{row}}\left(\frac{((W_{t,Q,0}^{(a)}+T)H_{t,Q,i})^\top\widetilde{K}_{t,i}^{(a)}}{\sqrt{d_{k,t}}}+C_{t,i}^M\right).
\]
Since every other head, every residual contribution, and every downstream stage is fixed, the full stage-$t$ output on sample $i$ is equal to a fixed matrix plus
\[
W_{t,O}^{(a)}\widetilde{V}_{t,i}^{(a)}A_{t,Q,i}^{(a)}(T)^\top.
\]
Hence, the conditional loss depends on $T$ only through the collection
\[
\left\{W_{t,O}^{(a)}\widetilde{V}_{t,i}^{(a)}A_{t,Q,i}^{(a)}(T)^\top\right\}_{i=1}^N.
\]

Define
\[
P_{t,i}^\perp\coloneqq I_{\varrho_{t,K,i}}-\frac{1}{\varrho_{t,K,i}}\mathbf{1}_{\varrho_{t,K,i}}\mathbf{1}_{\varrho_{t,K,i}}^\top;
K_{t,\mathrm{ctr}}^{(a)}\coloneqq[\widetilde{K}_{t,1}^{(a)}P_{t,1}^\perp,\ldots,\widetilde{K}_{t,N}^{(a)}P_{t,N}^\perp];
\widehat{T}_{t,Q}^{(a)}\coloneqq P_{K_{t,\mathrm{ctr}}^{(a)}}T_{t,Q}^{(a)}P_{H_{t,Q}}.
\]
By Lemma~\ref{lem:softmax-row-center}, for every $i$, the matrix $A_{t,Q,i}^{(a)}(T)$ depends only on
\[
\left(\frac{((W_{t,Q,0}^{(a)}+T)H_{t,Q,i})^\top\widetilde{K}_{t,i}^{(a)}}{\sqrt{d_{k,t}}}\right)P_{t,i}^\perp.
\]
Moreover, $(H_{t,Q,i})^\top(\widehat{T}_{t,Q}^{(a)})^\top\widetilde{K}_{t,i}^{(a)}P_{t,i}^\perp=(H_{t,Q,i})^\top(T_{t,Q}^{(a)})^\top\widetilde{K}_{t,i}^{(a)}P_{t,i}^\perp$, because $(H_{t,Q,i})^\top P_{H_{t,Q}}=(H_{t,Q,i})^\top\text{ and }P_{K_{t,\mathrm{ctr}}^{(a)}}\widetilde{K}_{t,i}^{(a)}P_{t,i}^\perp=\widetilde{K}_{t,i}^{(a)}P_{t,i}^\perp$.
Therefore $A_{t,Q,i}^{(a)}(\widehat{T}_{t,Q}^{(a)})=A_{t,Q,i}^{(a)}(T_{t,Q}^{(a)})\ \forall i$, so the conditional loss is invariant under
\[
T_{t,Q}^{(a)}\mapsto P_{K_{t,\mathrm{ctr}}^{(a)}}T_{t,Q}^{(a)}P_{H_{t,Q}}.
\]

Next, consider the case $S=K$.
Hold all blocks other than $T_{t,K}^{(a)}$ fixed and define
\[
\widetilde{Q}_{t,i}^{(a)}\coloneqq(W_{t,Q,0}^{(a)}+T_{t,Q}^{(a)})H_{t,Q,i},\quad \widetilde{V}_{t,i}^{(a)}\coloneqq(W_{t,V,0}^{(a)}+T_{t,V}^{(a)})H_{t,V,i}.
\]
For $T\in\R^{d_{k,t}\times n_{t,K}}$, set
\[
A_{t,K,i}^{(a)}(T)\coloneqq\mathrm{softmax}_{\mathrm{row}}\left(\frac{(\widetilde{Q}_{t,i}^{(a)})^\top(W_{t,K,0}^{(a)}+T)H_{t,K,i}}{\sqrt{d_{k,t}}}+C_{t,i}^M\right).
\]
As above, the conditional loss depends on $T$ only through
\[
\left\{W_{t,O}^{(a)}\widetilde{V}_{t,i}^{(a)}A_{t,K,i}^{(a)}(T)^\top\right\}_{i=1}^N.
\]
By Lemma~\ref{lem:softmax-row-center}, $A_{t,K,i}^{(a)}(T)$ depends only on
\[
\frac{(\widetilde{Q}_{t,i}^{(a)})^\top W_{t,K,0}^{(a)}H_{t,K,i}P_{t,i}^\perp+(\widetilde{Q}_{t,i}^{(a)})^\top TH_{t,K,i}P_{t,i}^\perp}{\sqrt{d_{k,t}}}.
\]
Define
\[
Q_{t,\mathrm{all}}^{(a)}\coloneqq[\widetilde{Q}_{t,1}^{(a)},\ldots,\widetilde{Q}_{t,N}^{(a)}], \
H_{t,K,\mathrm{ctr}}\coloneqq[H_{t,K,1}P_{t,1}^\perp,\ldots,H_{t,K,N}P_{t,N}^\perp],
\ 
\widehat{T}_{t,K}^{(a)}\coloneqq P_{Q_{t,\mathrm{all}}^{(a)}}T_{t,K}^{(a)}P_{H_{t,K,\mathrm{ctr}}}.
\]
Then $(\widetilde{Q}_{t,i}^{(a)})^\top\widehat{T}_{t,K}^{(a)}H_{t,K,i}P_{t,i}^\perp=(\widetilde{Q}_{t,i}^{(a)})^\top T_{t,K}^{(a)}H_{t,K,i}P_{t,i}^\perp$, because $(\widetilde{Q}_{t,i}^{(a)})^\top P_{Q_{t,\mathrm{all}}^{(a)}}=(\widetilde{Q}_{t,i}^{(a)})^\top,P_{H_{t,K,\mathrm{ctr}}}H_{t,K,i}P_{t,i}^\perp=H_{t,K,i}P_{t,i}^\perp$.
Thus $A_{t,K,i}^{(a)}(\widehat{T}_{t,K}^{(a)})=A_{t,K,i}^{(a)}(T_{t,K}^{(a)})\ \forall i$, and the conditional loss is invariant under
\[
T_{t,K}^{(a)}\mapsto P_{Q_{t,\mathrm{all}}^{(a)}}T_{t,K}^{(a)}P_{H_{t,K,\mathrm{ctr}}}.
\]

Finally, consider the case $S=V$.
Hold all blocks other than $T_{t,V}^{(a)}$ fixed.
Then all attention matrices
\[
A_{t,i}^{(a)}\coloneqq\mathrm{softmax}_{\mathrm{row}}\left(\frac{(\widetilde{Q}_{t,i}^{(a)})^\top\widetilde{K}_{t,i}^{(a)}}{\sqrt{d_{k,t}}}+C_{t,i}^M\right)
\]
are fixed, and the conditional loss depends on $T$ only through
\[
\left\{W_{t,O}^{(a)}(W_{t,V,0}^{(a)}+T)H_{t,V,i}(A_{t,i}^{(a)})^\top\right\}_{i=1}^N.
\]
Define $H_{t,V,A}^{(a)}\coloneqq[H_{t,V,1}(A_{t,1}^{(a)})^\top,\ldots,H_{t,V,N}(A_{t,N}^{(a)})^\top]$ and set $\widehat{T}_{t,V}^{(a)}\coloneqq P_{W_{t,O}^{(a)\top}}T_{t,V}^{(a)}P_{H_{t,V,A}^{(a)}}$.
Then for every $i$, $W_{t,O}^{(a)}\widehat{T}_{t,V}^{(a)}H_{t,V,i}(A_{t,i}^{(a)})^\top=W_{t,O}^{(a)}T_{t,V}^{(a)}H_{t,V,i}(A_{t,i}^{(a)})^\top$, because $W_{t,O}^{(a)}P_{W_{t,O}^{(a)\top}}=W_{t,O}^{(a)},P_{H_{t,V,A}^{(a)}}H_{t,V,i}(A_{t,i}^{(a)})^\top=H_{t,V,i}(A_{t,i}^{(a)})^\top$.
Hence, the conditional loss is invariant under
\[
T_{t,V}^{(a)}\mapsto P_{W_{t,O}^{(a)\top}}T_{t,V}^{(a)}P_{H_{t,V,A}^{(a)}}.
\]

We have therefore verified the hypotheses of Theorem~\ref{thm:rank-full-selfconsistent} for the flattened block family $\widetilde{\mathcal{S}}$.
Hence, there exists a global minimizer minimizing total rank over the global minimizer set, and for that minimizer, the rank of each adapted block is bounded by the ranks of the corresponding projector pairs.

Evaluating the projector ranks at this rank-minimal global minimizer gives the following intermediate bounds:
\begin{align*}
\rank(T_{t,Q}^{(a)\star})&\le\min\{\rank(K_{t,\mathrm{ctr}}^{(a)\star}),\rank(H_{t,Q}^\star)\} \\  \rank(T_{t,K}^{(a)\star})& \le\min\{\rank(Q_{t,\mathrm{all}}^{(a)\star}),\rank(H_{t,K,\mathrm{ctr}}^\star)\},
\end{align*}
and
\[
\rank(T_{t,V}^{(a)\star})\le\min\{\rank(W_{t,O}^{(a)}),\rank(H_{t,V,A}^{(a)\star})\}.
\]
Here
\[
K_{t,\mathrm{ctr}}^{(a)\star}\coloneqq[\widetilde{K}_{t,1}^{(a)\star}P_{t,1}^\perp,\ldots,\widetilde{K}_{t,N}^{(a)\star}P_{t,N}^\perp],
\quad 
Q_{t,\mathrm{all}}^{(a)\star}\coloneqq[\widetilde{Q}_{t,1}^{(a)\star},\ldots,\widetilde{Q}_{t,N}^{(a)\star}],
\]
and
\[
H_{t,V,A}^{(a)\star}\coloneqq[H_{t,V,1}^\star(A_{t,1}^{(a)\star})^\top,\ldots,H_{t,V,N}^\star(A_{t,N}^{(a)\star})^\top],
\quad 
H_{t,K,\mathrm{ctr}}^\star\coloneqq[H_{t,K,1}^\star P_{t,1}^\perp,\ldots,H_{t,K,N}^\star P_{t,N}^\perp],
\]

We now bound the ranks of these matrices.
Since 
\[
K_{t,\mathrm{ctr}}^{(a)\star}=[\widetilde{K}_{t,1}^{(a)\star},\ldots,\widetilde{K}_{t,N}^{(a)\star}]\operatorname{blkdiag}(P_{t,1}^\perp,\ldots,P_{t,N}^\perp),
\]
we have
\[
\rank(K_{t,\mathrm{ctr}}^{(a)\star})\le\min\left\{\rank([\widetilde{K}_{t,1}^{(a)\star},\ldots,\widetilde{K}_{t,N}^{(a)\star}]),\sum_{i=1}^N\rank(P_{t,i}^\perp)\right\}.
\]
Because $\rank(P_{t,i}^\perp)=\varrho_{t,K,i}-1$ and $[\widetilde{K}_{t,1}^{(a)\star},\ldots,\widetilde{K}_{t,N}^{(a)\star}]=(W_{t,K,0}^{(a)}+T_{t,K}^{(a)\star})H_{t,K}^\star$, it follows that
\[
\rank(K_{t,\mathrm{ctr}}^{(a)\star})\le\min\left\{d_{k,t},\rank(H_{t,K}^\star),\sum_{i=1}^N(\varrho_{t,K,i}-1)\right\}.
\]
Combining this with
\[
\rank(T_{t,Q}^{(a)\star})\le\min\{\rank(K_{t,\mathrm{ctr}}^{(a)\star}),\rank(H_{t,Q}^\star)\}
\]
gives the stated query bound.

For the key bound, first observe that $\rank(Q_{t,\mathrm{all}}^{(a)\star})\le\min\{d_{k,t},\rank(H_{t,Q}^\star)\}$, because $Q_{t,\mathrm{all}}^{(a)\star}=(W_{t,Q,0}^{(a)}+T_{t,Q}^{(a)\star})H_{t,Q}^\star$.
Also, $H_{t,K,\mathrm{ctr}}^\star=H_{t,K}^\star\operatorname{blkdiag}(P_{t,1}^\perp,\ldots,P_{t,N}^\perp)$, so
\[
\rank(H_{t,K,\mathrm{ctr}}^\star)\le\min\left\{\rank(H_{t,K}^\star),\sum_{i=1}^N(\varrho_{t,K,i}-1)\right\}.
\]
Combining these two estimates with
\[
\rank(T_{t,K}^{(a)\star})\le\min\{\rank(Q_{t,\mathrm{all}}^{(a)\star}),\rank(H_{t,K,\mathrm{ctr}}^\star)\}
\]
gives the stated key bound.

For the value bound, we have $H_{t,V,A}^{(a)\star}=H_{t,V}^\star\operatorname{blkdiag}\bigl((A_{t,1}^{(a)\star})^\top,\ldots,(A_{t,N}^{(a)\star})^\top\bigr)$, and hence
\[
\rank(H_{t,V,A}^{(a)\star})\le\min\left\{\rank(H_{t,V}^\star),\sum_{i=1}^N\rank(A_{t,i}^{(a)\star})\right\}.
\]
Since $W_{t,O}^{(a)}$ has at most $d_{v,t}$ columns, $\rank(W_{t,O}^{(a)})\le d_{v,t}$ so combining this with
\[
\rank(T_{t,V}^{(a)\star})\le\min\{\rank(W_{t,O}^{(a)}),\rank(H_{t,V,A}^{(a)\star})\}
\]
gives the stated value bound.
\end{proof}

The adapted-output extension follows from the same blockwise compression, but not from a simultaneous linear observation map for the full $Q/K/V/O$ tuple.

\begin{proof}[Proof of Theorem~\ref{thm:rank-attn-adapted-output}]
Let $\mathcal{I}$ index all selected query, key, value, and full output-projection updates. By the assumed continuity of the forward computation, the empirical loss is lower semicontinuous and bounded below. Since every adapted block has a strictly positive nuclear-norm coefficient, the objective is coercive on the finite-dimensional product space $\prod_{j\in\mathcal{I}}\R^{m_j\times n_j}$ and therefore attains a global minimum.

Since the sum of matrix ranks is integer valued and bounded above by $\sum_{j\in\mathcal{I}}\min\{m_j,n_j\}$, there is a global minimizer $\{T_j^\star\}$ minimizing that sum over all global minimizers. We establish the hypotheses of Theorem~\ref{thm:rank-full-selfconsistent} by fixing all companion updates each time a block is compressed.

For any query update, the input matrices $H_{t,Q,i}$ and adapted keys $\widetilde K_{t,i}^{(a)}$ do not depend on that query update.
Row-wise score centering and the argument in the proof of Theorem~\ref{thm:rank-attn} therefore show that replacing $T_{t,Q}^{(a)}$ by $P_{K_{t,\mathrm{ctr}}^{(a)}}T_{t,Q}^{(a)}P_{H_{t,Q}}$ preserves every corresponding attention matrix.
The adapted output projection is held fixed during this replacement, so it does not change that invariance. Similarly, for a key update, replacing $T_{t,K}^{(a)}$ by
\[
P_{Q_{t,\mathrm{all}}^{(a)}}T_{t,K}^{(a)}P_{H_{t,K,\mathrm{ctr}}}
\]
preserves the centered score products and the attention matrices. Thus the query and key projector bounds from Theorem~\ref{thm:rank-attn} remain valid.

For an adapted value block, all attention matrices and the companion output update $T_{t,O}$ are fixed.
Write the adapted output projection in head blocks as
\[
\widetilde W_{t,O}\coloneqq W_{t,O,0}+T_{t,O}=[\widetilde W_{t,O}^{(a)}]_{a\in\mathcal{A}_t},\qquad
\widetilde W_{t,O}^{(a)}\in\R^{d_{\mathrm{out},t}\times d_{v,t}},
\]
and put
\[
C_{t,i}^{(a)}=H_{t,V,i}(A_{t,i}^{(a)})^\top,
\quad C_t^{(a)}=[C_{t,1}^{(a)},\ldots,C_{t,N}^{(a)}].
\]
Here $C_{t,i}^{(a)}\in\R^{n_{t,V}\times\varrho_{t,Q,i}}$. Set
\[
P_O^{(a)}=P_{(\widetilde W_{t,O}^{(a)})^\top},\qquad
P_C^{(a)}=P_{C_t^{(a)}}.
\]
For any $T\in\R^{d_{v,t}\times n_{t,V}}$ and every sample $i$,
\[
\widetilde W_{t,O}^{(a)}P_O^{(a)}TP_C^{(a)}C_{t,i}^{(a)}
=\widetilde W_{t,O}^{(a)}TC_{t,i}^{(a)}.
\]
Thus the stage output and the downstream loss are unchanged by $T\mapsto P_O^{(a)}TP_C^{(a)}$. The corresponding nuclear norm does not increase by Lemma~\ref{lem:proj-compression}.

For the full output projection, hold all $Q/K/V$ updates and all other blocks fixed. For each training sample $i$, define
\[
B_{t,i}^{(a)}\coloneqq\widetilde V_{t,i}^{(a)}(A_{t,i}^{(a)})^\top
\in\R^{d_{v,t}\times\varrho_{t,Q,i}},\quad
B_{t,i}\coloneqq\begin{bmatrix}B_{t,i}^{(a_1)}\\ \vdots\\ B_{t,i}^{(a_{|\mathcal{A}_t|})}\end{bmatrix},\quad
H_{t,O}\coloneqq[B_{t,1},\ldots,B_{t,N}].
\]
Every $B_{t,i}$ is then fixed, and the varying part of the stage output equals $T_{t,O}B_{t,i}$. Since $P_{H_{t,O}}B_{t,i}=B_{t,i}$,
\[
(T_{t,O}P_{H_{t,O}})B_{t,i}=T_{t,O}B_{t,i}\qquad(i\in[N]).
\]
The replacement $T_{t,O}\mapsto T_{t,O}P_{H_{t,O}}$ therefore preserves the objective's loss term and does not increase the nuclear norm.

All conditional compression projectors depend only on companion updates. Theorem~\ref{thm:rank-full-selfconsistent} applies to the finite block family $\mathcal{I}$ and gives, at the chosen rank-minimal global minimizer,
\begin{equation}\label{eqn:adapted-output-sharp}
\rank(T_{t,V}^{(a)\star})\le\min\{\rank(\widetilde W_{t,O}^{(a)\star}),\rank(C_t^{(a)\star})\},
\quad
\rank(T_{t,O}^\star)\le\min\{d_{\mathrm{out},t},\rank(H_{t,O}^\star)\}.
\end{equation}
It remains to bound the three ranks on the right-hand side of~\eqref{eqn:adapted-output-sharp}.
First, $\widetilde W_{t,O}^{(a)\star}$ has $d_{\mathrm{out},t}$ rows and $d_{v,t}$ columns, so
\[
\rank(\widetilde W_{t,O}^{(a)\star})\le\min\{d_{\mathrm{out},t},d_{v,t}\}\le d_{v,t}.
\]
Second, the identity
\[
C_t^{(a)\star}=H_{t,V}^\star\operatorname{blkdiag}((A_{t,1}^{(a)\star})^\top,\ldots,(A_{t,N}^{(a)\star})^\top)
\]
implies $\rank(C_t^{(a)\star})\le\min\{\rank(H_{t,V}^\star),\sum_{i=1}^N\rank(A_{t,i}^{(a)\star})\}$.
Third, $H_{t,O}^\star\in\R^{(|\mathcal{A}_t|d_{v,t})\times\sum_{i=1}^N\varrho_{t,Q,i}}$ by construction, so $\rank(H_{t,O}^\star)\le\min\{|\mathcal{A}_t|d_{v,t},\sum_{i=1}^N\varrho_{t,Q,i}\}$.
Substituting these three estimates into~\eqref{eqn:adapted-output-sharp} yields the two bounds stated in the theorem and completes the proof.
\end{proof}

\subsection{Proofs for Section~\ref{sec:rankdef}}\label{sec:proof-rankdef}

For the positive rank-deficiency statement, the first step is to write the first-order stationarity equations in a symmetric block form.

\begin{lemma}[Stationarity conditions for the perturbed block LoRA objective]\label{lem:KKT-perturbed-lora}
In the setting of Definition~\ref{def:perturbed-lora-new}, suppose $\mathcal{L}$ is differentiable and $\{(U_S,V_S)\}$ is a first-order stationary point of $\mathcal{L}^{\mathrm{lora}}_{\{\lambda_S\},\{P_S\}}$.
Let $T_S\coloneqq U_SV_S^\top$ and $G_S\coloneqq\nabla_{T_S}\mathcal{L}(\{T_{S'}\})$.
Then for every $S\in\mathcal{S}$,
\[
\bigl(\mathcal{B}_{m_S,n_S}(G_S)+\lambda_S I_{m_S+n_S}+P_S\bigr)R_S=0,R_S=\begin{bmatrix}U_S\\V_S\end{bmatrix}.
\]
\end{lemma}

\begin{proof}[Proof of Lemma~\ref{lem:KKT-perturbed-lora}]
Fix $S$ and differentiate w.r.t.\ $U_S$ and $V_S$ using the chain rule: $\nabla_{U_S}\mathcal{L}(\{U_{S'}V_{S'}^\top\})=G_SV_S$ and $\nabla_{V_S}\mathcal{L}(\{U_{S'}V_{S'}^\top\})=G_S^\top U_S$.
The regularizer contributes $\lambda_SU_S$ and $\lambda_SV_S$.
The perturbation term satisfies $\nabla_{R_S}\frac12\langle P_S,R_SR_S^\top\rangle=P_SR_S$.
Stacking the two stationarity equations yields the displayed KKT condition.
\end{proof}

The genericity argument is a dimension-counting statement.
We will need the following elementary measure-zero lemma:

\begin{lemma}[Measure-zero Minkowski sums with $C^1$ images]\label{lem:measure-zero-sum}
Let $\mathcal{M}$ be an $m$-dimensional embedded $C^1$ submanifold of $\R^d$.
Let $U\subseteq\R^{p}$ be open, and let $f:U\to\R^d$ be $C^1$.
If $m+p<d$, then
\[
\mathcal{M}+f(U)\coloneqq\{x+f(u):x\in\mathcal{M},u\in U\}
\]
has Lebesgue measure zero in $\R^d$.
\end{lemma}

\begin{proof}[Proof of Lemma~\ref{lem:measure-zero-sum}]
Since $\mathcal{M}\subseteq\R^d$ is an embedded submanifold, it is second countable and admits a countable $C^1$ atlas.
It is therefore enough to prove the claim on one coordinate patch.

Fix a chart $\varphi:V\to\mathcal{M}$, where $V\subseteq\R^{m}$ is open.
On this patch, $\{\varphi(v)+f(u):v\in V,u\in U\}$ is the image of the $C^1$ map
\[
F:V\times U\to\R^d,\qquad F(v,u)=\varphi(v)+f(u).
\]
The domain $V\times U$ is open in $\R^{m+p}$.
Set $k\coloneqq m+p$.
By assumption, $k<d$.

Cover $V\times U$ by countably many compact cubes $Q_{\ell}$ whose interiors cover $V\times U$ and whose closures are contained in $V\times U$.
Fix one such cube $Q$.
Since $F$ is $C^1$, its derivative is bounded on $Q$.
Hence there exists $L_Q<\infty$ such that
\[
\|F(x)-F(y)\|_2\le L_Q\|x-y\|_2\ \forall x,y\in Q.
\]
Subdivide $Q$ into $N^{k}$ subcubes of side length proportional to $N^{-1}$.
The image of each subcube is contained in a Euclidean ball in $\R^d$ of radius at most $C_QN^{-1}$, where $C_Q$ is independent of $N$.
Therefore the $d$-dimensional Lebesgue outer measure of $F(Q)$ is bounded by
\[
N^{k}\omega_d\left(\frac{C_Q}{N}\right)^d=\omega_dC_Q^dN^{k-d},
\]
where $\omega_d$ is the volume of the unit ball in $\R^d$.
Since $k-d<0$, this upper bound converges to $0$ as $N\to\infty$.
Thus $F(Q)$ has Lebesgue measure zero.

Taking the countable union over all cubes $Q_{\ell}$ and then over the countable atlas of $\mathcal{M}$ shows that $\mathcal{M}+f(U)$ has Lebesgue measure zero in $\R^d$.
\end{proof}

\begin{proof}[Proof of Theorem~\ref{thm:generic-rankdef-master}]
If $r>d$, then every $R\in\R^{d\times r}$ satisfies $\rank(R)\le d<r$, so the conclusion is immediate.
Assume from now on that $r\le d$.

Let
\[
D_{\mathrm{sym}}\coloneqq\dim\mathbb{S}^d=\frac{d(d+1)}{2}.
\]
For $s\in\{r,r+1,\ldots,d\}$, define $\mathcal{A}_s\coloneqq\{M\in\mathbb{S}^d:\dim\mathcal{N}(M)=s\}$.
The set $\mathcal{A}_s$ is an embedded $C^\infty$ submanifold of $\mathbb{S}^d$ with
\[
\dim(\mathcal{A}_s)=D_{\mathrm{sym}}-\frac{s(s+1)}{2}.
\]
Indeed, $\mathcal{A}_s$ is the set of symmetric matrices of rank $d-s$; its codimension in $\mathbb{S}^d$ is the dimension of $\mathbb{S}^s$, namely $s(s+1)/2$.

Suppose that for some $P\in\mathbb{S}^d$ there exist $G\in\mathcal{G}$ and $R\in\R^{d\times r}$ such that $\rank(R)=r$ and
\[
(G+\lambda I_d+P)R=0.
\]
Set $M\coloneqq G+\lambda I_d+P$.
Since the columns of $R$ are linearly independent and lie in $\mathcal{N}(M)$, we have $\dim\mathcal{N}(M)\ge r$.
Therefore $M\in\bigcup_{s=r}^d\mathcal{A}_s$, and hence
\[
P=M-\lambda I_d-G\in\bigcup_{s=r}^d\bigl((\mathcal{A}_s-\lambda I_d)-\mathcal{G}\bigr).
\]

It remains to show that each set $(\mathcal{A}_s-\lambda I_d)-\mathcal{G}$ has Lebesgue measure zero in $\mathbb{S}^d$.
Since $-\mathcal{G}$ is contained in a finite union of $C^1$ images of dimension at most $p$, there exist open sets $U_j\subseteq\R^{p_j}$ with $p_j\le p$ and $C^1$ maps $f_j:U_j\to\mathbb{S}^d,j=1,\ldots,J$, such that
\[
-\mathcal{G}\subseteq\bigcup_{j=1}^J f_j(U_j).
\]
Fix $s\in\{r,\ldots,d\}$ and $j\in[J]$.
The translated manifold $\mathcal{A}_s-\lambda I_d$ is an embedded $C^\infty$ submanifold of $\mathbb{S}^d$ with dimension $D_{\mathrm{sym}}-\frac{s(s+1)}{2}$.
Because $s\ge r$ and $p_j\le p$,
\[
\left(D_{\mathrm{sym}}-\frac{s(s+1)}{2}\right)+p_j\le D_{\mathrm{sym}}-\frac{r(r+1)}{2}+p<
D_{\mathrm{sym}}.
\]
Lemma~\ref{lem:measure-zero-sum}, applied inside the Euclidean space $\mathbb{S}^d$, implies that $(\mathcal{A}_s-\lambda I_d)+f_j(U_j)$ has Lebesgue measure zero in $\mathbb{S}^d$.
Taking the finite union over $j$ gives that $(\mathcal{A}_s-\lambda I_d)-\mathcal{G}$ has Lebesgue measure zero in $\mathbb{S}^d$.

Finally, taking the finite union over $s=r,\ldots,d$ gives a Lebesgue-null exceptional set.
Thus, for every $P$ outside this exceptional set, there do not exist $G\in\mathcal{G}$ and $R\in\R^{d\times r}$ with $\rank(R)=r$ satisfying
\[
(G+\lambda I_d+P)R=0.
\]
Intersecting this full-measure set with $\mathbb{S}_+^d$ proves the conclusion for Lebesgue-a.e. positive semidefinite perturbation.
\end{proof}
The next lemma connects the stagewise observation model to the KKT matrices above.
\begin{lemma}[Stagewise gradient range under $L$-stage linear map dependence]\label{lem:grad-range-L-stage}
Assume $\mathcal{L}$ from Definition~\ref{def:L-stage-block-lin-meas} is differentiable.
Fix a parameter tuple and the corresponding states $z_t$.
Then for every $t\in[L]$ and $S\in\mathcal{S}_t$, we have $\nabla_{T_{t,S}}\mathcal{L}\in\mathcal{R}(\Phi_{t,S}(\cdot,z_{t+1})^\star)$.
\end{lemma}

\begin{proof}[Proof of Lemma~\ref{lem:grad-range-L-stage}]
Fix $t\in[L]$ and $S\in\mathcal{S}_t$, and fix all blocks except $T_{t,S}$.
Because all upstream blocks are fixed, the state $z_{t+1}$ is fixed.
Define the conditional loss
\[
F_{t,S}(T)\coloneqq\mathcal{L}\bigl(\{T_{u,A}\}_{(u,A)\ne(t,S)},T_{t,S}=T\bigr).
\]
By Definition~\ref{def:L-stage-block-lin-meas}, the function $F_{t,S}$ depends on $T$ only through the linear map $T\mapsto\Phi_{t,S}(T,z_{t+1})$.
Therefore, for every $\Delta\in\mathcal{N}(\Phi_{t,S}(\cdot,z_{t+1}))$ and every $T$, $F_{t,S}(T+\Delta)=F_{t,S}(T)$.
Since $F$ is differentiable, this implies
\[
0=DF_{t,S}(T)[\Delta]=\langle\nabla_{T_{t,S}}\mathcal{L},\Delta\rangle_F\forall\Delta\in\mathcal{N}(\Phi_{t,S}(\cdot,z_{t+1})).
\]
Hence $\nabla_{T_{t,S}}\mathcal{L}\in\mathcal{N}(\Phi_{t,S}(\cdot,z_{t+1}))^\perp$.
In finite dimensions, $\mathcal{N}(\Phi_{t,S}(\cdot,z_{t+1}))^\perp=\mathcal{R}(\Phi_{t,S}(\cdot,z_{t+1})^\star)$, so $\nabla_{T_{t,S}}\mathcal{L}\in\mathcal{R}(\Phi_{t,S}(\cdot,z_{t+1})^\star)$.
\end{proof}

To apply Theorem~\ref{thm:generic-rankdef-master}, we must bound the dimension of the union of these moving ranges over all reachable upstream states. The next lemma provides this bound by recursively parameterizing the reachable observed-output sets.

\begin{lemma}[Recursive parameterization of reachable states and moving adjoint ranges]\label{lem:reach-move-param-L}
Assume Assumption~\ref{ast:multistage-c1} holds.
Then the following statements hold:

\begin{enumerate}
\item For each $t\in[L+1]$, the reachable state set $\mathcal{Z}_t^{\mathrm{reach}}$ is $C^1$-parameterizable with dimension at most $\beta_t$.

\item For each $t\in[L]$ and $S\in\mathcal{S}_t$, define
\[
\mathcal{G}_{t,S}\coloneqq\bigcup_{z_{t+1}\in\mathcal{Z}_{t+1}^{\mathrm{reach}}}\left\{\mathcal{B}_{m_{t,S},n_{t,S}}(G):G\in\mathcal{R}\bigl(\Phi_{t,S}(\cdot,z_{t+1})^\star\bigr)\right\}\subseteq\mathbb{S}^{m_{t,S}+n_{t,S}}.
\]
Then $\mathcal{G}_{t,S}$ is $C^1$-parameterizable with dimension at most $\beta_{t+1}+d_{t,S}^{\max}$.
\end{enumerate}
\end{lemma}

\begin{proof}[Proof of Lemma~\ref{lem:reach-move-param-L}]
For each $t\in[L]$, let $\mathcal{O}_t$ be an open neighborhood of the reachable argument set
\[
\mathcal{A}_t^{\mathrm{arg}}\coloneqq\left\{\left(\{u_S\}_{S\in\mathcal{S}_t},z_{t+1}\right):z_{t+1}\in\mathcal{Z}_{t+1}^{\mathrm{reach}},u_S\in\mathcal{R}\bigl(\Phi_{t,S}(\cdot,z_{t+1})\bigr)\right\}
\]
on which $g_t$ is $C^1$.
For each pair $(t,S)$, let $\mathcal{O}_{t,S}^{\Phi}$ be an open neighborhood of $\mathcal{Z}_{t+1}^{\mathrm{reach}}$ on which the coordinate maps
\[
z_{t+1}\mapsto\Phi_{t,S}(B,z_{t+1})\text{ and }z_{t+1}\mapsto\Phi_{t,S}(\cdot,z_{t+1})^\star(e)
\]
are $C^1$ for every basis matrix $B$ of $\R^{m_{t,S}\times n_{t,S}}$ and every basis vector $e$ of $\mathcal{E}_{t,S}$.
These open sets exist by Assumption~\ref{ast:multistage-c1}.

We first prove the reachable-state claim by backward induction on $t$.
For $t=L+1$, the set $\mathcal{Z}_{L+1}^{\mathrm{reach}}$ is a singleton, hence it is $C^1$-parameterizable with dimension $0=\beta_{L+1}$.

Assume that $\mathcal{Z}_{t+1}^{\mathrm{reach}}$ is $C^1$-parameterizable with dimension at most $\beta_{t+1}$.
Choose finitely many $C^1$ maps $z_{t+1}^{(j)}:\Theta_j\to\mathcal{Z}_{t+1}^{\mathrm{reach}},j=1,\ldots,J$, whose images cover $\mathcal{Z}_{t+1}^{\mathrm{reach}}$, where each $\Theta_j$ is open and has dimension at most $\beta_{t+1}$.
For each block $S\in\mathcal{S}_t$, choose a basis $B_{S,1},\ldots,B_{S,p_S}$ of $\R^{m_{t,S}\times n_{t,S}}$, and set $k_S\coloneqq d_{t,S}^{\max}$.
For every $j$ and every tuple of index sets $I=(I_S)_{S\in\mathcal{S}_t},I_S=\{i_{S,1},\ldots,i_{S,q_S}\}\subseteq[p_S],q_S\le k_S$, define the preliminary argument map
\[
Q_{j,I}\left(\theta,\{\alpha_S\}_{S\in\mathcal{S}_t}\right)\coloneqq\left(\left\{\sum_{\ell=1}^{q_S}\alpha_{S,\ell}\Phi_{t,S}\bigl(B_{S,i_{S,\ell}},z_{t+1}^{(j)}(\theta)\bigr)\right\}_{S\in\mathcal{S}_t},z_{t+1}^{(j)}(\theta)\right).
\]
Here $\alpha_S\in\R^{q_S}$.
Restrict the domain to the open set $\Omega_{j,I}\coloneqq Q_{j,I}^{-1}(\mathcal{O}_t)$.
On $\Omega_{j,I}$, define
\[
F_{j,I}\left(\theta,\{\alpha_S\}_{S\in\mathcal{S}_t}\right)\coloneqq g_t\left(\left\{\sum_{\ell=1}^{q_S}\alpha_{S,\ell}\Phi_{t,S}\bigl(B_{S,i_{S,\ell}},z_{t+1}^{(j)}(\theta)\bigr)\right\}_{S\in\mathcal{S}_t},z_{t+1}^{(j)}(\theta)\right).
\]
The restriction to $\Omega_{j,I}$ is the important point: it ensures that the argument of $g_t$ lies in the neighborhood on which $g_t$ is assumed to be $C^1$.
Since each $z_{t+1}^{(j)}(\theta)$ lies in $\mathcal{Z}_{t+1}^{\mathrm{reach}}$, and since the coordinate maps of $\Phi_{t,S}$ are $C^1$ on $\mathcal{O}_{t,S}^{\Phi}$, the map $F_{j,I}$ is $C^1$ on $\Omega_{j,I}$.

The dimension of the domain of $F_{j,I}$ is at most
\[
\beta_{t+1}+\sum_{S\in\mathcal{S}_t}q_S\le\beta_{t+1}+\sum_{S\in\mathcal{S}_t}d_{t,S}^{\max}=\beta_t.
\]
It remains to check that these finitely many maps cover $\mathcal{Z}_t^{\mathrm{reach}}$.
Take any $z_t\in\mathcal{Z}_t^{\mathrm{reach}}$.
By definition, there exist $z_{t+1}\in\mathcal{Z}_{t+1}^{\mathrm{reach}},u_S\in\mathcal{R}\bigl(\Phi_{t,S}(\cdot,z_{t+1})\bigr)$ such that
\[
z_t=g_t(\{u_S\}_{S\in\mathcal{S}_t},z_{t+1}).
\]
Choose $j$ and $\theta\in\Theta_j$ with $z_{t+1}=z_{t+1}^{(j)}(\theta)$.
For each $S$, the range $
\mathcal{R}\bigl(\Phi_{t,S}(\cdot,z_{t+1})\bigr)$ has dimension at most $k_S$ and is spanned by the finite list
\[
\left\{\Phi_{t,S}(B_{S,a},z_{t+1})\right\}_{a=1}^{p_S}.
\]
Therefore, $u_S$ can be expressed as a linear combination of at most $k_S$ vectors from this list.
This gives some tuple $I$ and coefficients $\{\alpha_S\}$ such that $z_t=F_{j,I}\left(\theta,\{\alpha_S\}_{S\in\mathcal{S}_t}\right)$.
Moreover, the corresponding argument belongs to $\mathcal{A}_t^{\mathrm{arg}}\subseteq\mathcal{O}_t$, so the parameter point lies in $\Omega_{j,I}$.
Thus $\mathcal{Z}_t^{\mathrm{reach}}$ is covered by finitely many $C^1$ images of dimension at most $\beta_t$.
This completes the induction and proves the first claim.

We now prove the state-dependent adjoint ranges claim. 
Fix $t\in[L]$ and $S\in\mathcal{S}_t$.
By the first claim, $\mathcal{Z}_{t+1}^{\mathrm{reach}}$ is covered by finitely many $C^1$ parameterizations $z_{t+1}^{(j)}:\Theta_j\to\mathcal{Z}_{t+1}^{\mathrm{reach}}$, with domain dimension at most $\beta_{t+1}$.
Let $e_1,\ldots,e_{q_{t,S}}$ be a basis of $\mathcal{E}_{t,S}$, and set $k\coloneqq d_{t,S}^{\max}$.
For every $j$ and every index set $I=\{i_1,\ldots,i_q\}\subseteq[q_{t,S}],q\le k$, define
\[
M_{j,I}(\theta,\alpha)\coloneqq\sum_{\ell=1}^{q}\alpha_\ell\mathcal{B}_{m_{t,S},n_{t,S}}\left(\Phi_{t,S}(\cdot,z_{t+1}^{(j)}(\theta))^\star(e_{i_\ell})\right),
\]
where $\alpha\in\R^q$.
As before, one may restrict the domain to the open preimage
\[
\Omega_{j,I}^{\Phi}\coloneqq\left\{(\theta,\alpha):z_{t+1}^{(j)}(\theta)\in\mathcal{O}_{t,S}^{\Phi}\right\}.
\]
This restriction contains all parameter points needed to cover the reachable set, because $z_{t+1}^{(j)}(\Theta_j)\subseteq\mathcal{Z}_{t+1}^{\mathrm{reach}}\subseteq\mathcal{O}_{t,S}^{\Phi}$.
By Assumption~\ref{ast:multistage-c1}, each map $M_{j,I}$ is $C^1$ on $\Omega_{j,I}^{\Phi}$.
Its domain dimension is at most $\beta_{t+1}+q\le\beta_{t+1}+d_{t,S}^{\max}$.

It remains to verify coverage.
Take any element of $\mathcal{G}_{t,S}$.
It has the form $\mathcal{B}_{m_{t,S},n_{t,S}}(G),G\in\mathcal{R}\bigl(\Phi_{t,S}(\cdot,z_{t+1})^\star\bigr)$ for some $z_{t+1}\in\mathcal{Z}_{t+1}^{\mathrm{reach}}$.
Choose $j$ and $\theta$ such that $z_{t+1}=z_{t+1}^{(j)}(\theta)$.
The space $\mathcal{R}\bigl(\Phi_{t,S}(\cdot,z_{t+1})^\star\bigr)$ has dimension at most $d_{t,S}^{\max}=k$ and is spanned by
\[
\left\{\Phi_{t,S}(\cdot,z_{t+1})^\star(e_a)\right\}_{a=1}^{q_{t,S}}.
\]
Hence $G$ can be written as a linear combination of at most $k$ vectors from this list.
Therefore $\mathcal{B}_{m_{t,S},n_{t,S}}(G)$ belongs to one of the images of $M_{j,I}$.
Thus $\mathcal{G}_{t,S}$ is $C^1$-parameterizable with dimension at most $\beta_{t+1}+d_{t,S}^{\max}$.
\end{proof}

We now combine the KKT identity, the gradient-range lemma, and the parameterization lemma.
The result is the stagewise version of the master genericity theorem:
\begin{theorem}[Generic rank-deficiency for $L$-stage perturbed LoRA]\label{thm:rankdef-L-stage}
Assume the objective has $L$-stage blockwise linear map dependence, is differentiable, and Assumption~\ref{ast:multistage-c1} holds.
Fix factor ranks $r_{t,S}$ and regularization parameters $\lambda_{t,S}>0$.
If
\[
\frac{r_{t,S}(r_{t,S}+1)}{2}>\beta_{t+1}+d_{t,S}^{\max}\ \forall t\in[L],S\in\mathcal{S}_t,
\]
then for Lebesgue-a.e.\ choice of PSD perturbations $\{P_{t,S}\}$, every first-order stationary point of the perturbed LoRA objective satisfies
\[
\rank(R_{t,S})<r_{t,S}\forall t\in[L],S\in\mathcal{S}_t.
\]
\end{theorem}

\begin{proof}[Proof of Theorem~\ref{thm:rankdef-L-stage}]
Fix a stage/block $(t,S)$.
At any first-order stationary point, Lemma~\ref{lem:KKT-perturbed-lora} gives
\[
\bigl(\mathcal{B}_{m_{t,S},n_{t,S}}(G_{t,S})+\lambda_{t,S}I+P_{t,S}\bigr)R_{t,S}=0,
\]
where $G_{t,S}=\nabla_{T_{t,S}}\mathcal{L}$.
By Lemma~\ref{lem:grad-range-L-stage}, $G_{t,S}\in\mathcal{R}(\Phi_{t,S}(\cdot,z_{t+1})^\star)$ for the reachable state $z_{t+1}$ induced by the stationary point.
Hence $\mathcal{B}_{m_{t,S},n_{t,S}}(G_{t,S})\in\mathcal{G}_{t,S}$.
By Lemma~\ref{lem:reach-move-param-L}, the family $\mathcal{G}_{t,S}$ is $C^1$-parameterizable with dimension at most $\beta_{t+1}+d_{t,S}^{\max}$.
For the block $(t,S)$, define $d_{t,S}^{\mathrm{blk}}\coloneqq m_{t,S}+n_{t,S}$.
By Lemma~\ref{lem:reach-move-param-L}, the family $\mathcal{G}_{t,S}$ is contained in a finite union of $C^1$ images of dimension at most $\beta_{t+1}+d_{t,S}^{\max}$.
Since
\[
\frac{r_{t,S}(r_{t,S}+1)}{2}>\beta_{t+1}+d_{t,S}^{\max},
\]
Theorem~\ref{thm:generic-rankdef-master}, applied with $d=d_{t,S}^{\mathrm{blk}}$, $r=r_{t,S}$, $\lambda=\lambda_{t,S}$, and $\mathcal{G}=\mathcal{G}_{t,S}$, gives a full-measure set $\Omega_{t,S}\subseteq\mathbb{S}_+^{m_{t,S}+n_{t,S}}$ such that, whenever $P_{t,S}\in\Omega_{t,S}$, no matrix $R_{t,S}\in\R^{(m_{t,S}+n_{t,S})\times r_{t,S}}$ with $\rank(R_{t,S})=r_{t,S}$ can satisfy the block KKT equation
\[
\bigl(\mathcal{B}_{m_{t,S},n_{t,S}}(G)+\lambda_{t,S}I+P_{t,S}\bigr)R_{t,S}=0
\]
for any $\mathcal{B}_{m_{t,S},n_{t,S}}(G)\in\mathcal{G}_{t,S}$.

Taking the product of the full-measure sets $\Omega_{t,S}$ over the finite block family preserves full measure.
For any perturbation tuple in this product set, every first-order stationary point satisfies the displayed KKT equation for each block by Lemma~\ref{lem:KKT-perturbed-lora}, with its gradient block belonging to $\mathcal{G}_{t,S}$ by Lemma~\ref{lem:grad-range-L-stage}.
Therefore every first-order stationary point satisfies $\rank(R_{t,S})<r_{t,S}$ for every $t\in[L]$ and $S\in\mathcal{S}_t$.
\end{proof}
The attention corollary is obtained by evaluating the bottleneck dimension for the concrete map $T\mapsto TH_{t,S}(z_{t+1})$.
\begin{proof}[Proof of Corollary~\ref{cor:rankdef-L-attn}]
For an adapted attention projection block $\alpha=(a,S)$, the stagewise linear map has the form $\Phi_{t,\alpha}(T,z_{t+1})=TH_{t,S}(z_{t+1})$.
Therefore
\[
d_{t,\alpha}(z_{t+1})=\dim\mathcal{R}(\Phi_{t,\alpha}(\cdot,z_{t+1}))=m_{t,\alpha}\rank(H_{t,S}(z_{t+1})).
\]
Taking the supremum over $z_{t+1}\in\mathcal{Z}_{t+1}^{\mathrm{reach}}$ gives $d_{t,\alpha}^{\max}=m_{t,\alpha}r_{H,t,\alpha}^{\max}$.
Substituting this identity into Theorem~\ref{thm:rankdef-L-stage} gives the stated result.
\end{proof}

Now we prove the limitation result from Section~\ref{sec:rankdef}.

\begin{proof}[Proof of Proposition~\ref{prop:nogo-new}]
Choose any constant $a>\lambda+\varepsilon$ and set
\[
g(t)\coloneqq a(t^2-1)^2.
\]
Then $g$ is smooth and nonconvex, $g(0)=a$, $g(1)=0$, and $g'(0)=0$. 
Write $R=\begin{bmatrix}U\\V\end{bmatrix}$ with $U=(u_1,\ldots,u_r)$ and $V=(v_1,\ldots,v_r)$.
Let
\[
t(R)\coloneqq UV^\top=\sum_{k=1}^ru_kv_k.
\]
Under the column-wise vectorization $r'=\mathrm{vec}(R)\in\R^{2r}$, the chain rule gives
\[
\nabla_{r'}^2(g\circ t)(0)=g''(t(0))\nabla t(0)\nabla t(0)^\top+g'(t(0))\nabla^2t(0)=0,
\]
because $t(0)=0,\nabla t(0)=0$, and $g'(0)=0$.
Moreover,
\[
\nabla_{r'}^2\left(\frac{\lambda}{2}\|R\|_F^2\right)=\lambda I_{2r},\nabla_{r'}^2\left(\frac12\langle P,RR^\top\rangle\right)=I_r\otimes P.
\]
The perturbed objective is
\[
F_P(R)\coloneqq g(t(R))+\frac{\lambda}{2}\|R\|_F^2+\frac12\langle P,RR^\top\rangle.
\]
Therefore $\nabla_{r'}^2F_P(0)=\lambda I_{2r}+I_r\otimes P\succeq\lambda I_{2r}\succ0$, so $R=0$ is a strict local minimum.
Now set
\[
R_1\coloneqq
\begin{bmatrix}
1&0&\cdots&0\\
1&0&\cdots&0
\end{bmatrix}.
\]
Then
\[
t(R_1)=1,\|R_1\|_F^2=2,R_1R_1^\top=\begin{bmatrix}1&1\\1&1\end{bmatrix},\|R_1R_1^\top\|_F=2.
\]
Hence
\begin{align*}
F_P(R_1)&=g(1)+\frac{\lambda}{2}\|R_1\|_F^2+\frac12\langle P,R_1R_1^\top\rangle\\
&=\lambda+\frac12\langle P,R_1R_1^\top\rangle\\
&\le\lambda+\frac12\|P\|_F\|R_1R_1^\top\|_F\\
&=\lambda+\|P\|_F\\
&\le\lambda+\varepsilon<a=F_P(0).
\end{align*}
Thus, $R=0$ is not global.
\end{proof}

\subsection{Proofs for Section~\ref{sec:transfer-gen}}\label{sec:proof-transfer-gen}

\begin{lemma}[{\citep[Lemma~5.1]{recht2010guaranteed}}]\label{lem:regular}
For any $T\in\R^{m\times n}$ and any integer $r\ge\rank(T)$,
\[
\|T\|_\star=\frac{1}{2}\min_{\substack{U\in\R^{m\times r},V\in\R^{n\times r}\\UV^\top=T}}\left\{\|U\|_F^2+\|V\|_F^2\right\}.
\]
Equivalently,
\[
\|T\|_\star=\frac{1}{2}\min_{\substack{U\in\R^{m\times r},V\in\R^{n\times r}\\r\in\N,UV^\top=T}}\left\{\|U\|_F^2+\|V\|_F^2\right\}.
\]
\end{lemma}
We finally prove that small PSD perturbations do not substantially change the trace-norm objective value attained by a global factorized solution.
\begin{proof}[Proof of Theorem~\ref{thm:transfer}]
The perturbed LoRA objective is lower semicontinuous because $\mathcal{L}$ is lower semicontinuous and the maps $(U_j,V_j)\mapsto U_jV_j^\top$ are continuous. It is coercive because
\[
\sum_{j\in\mathcal{I}}\frac{\lambda_j}{2}\bigl(\|U_j\|_F^2+\|V_j\|_F^2\bigr)
\]
is coercive, and each perturbation term is nonnegative. Since $\mathcal{L}$ is bounded below, the perturbed LoRA objective attains a global minimum.

For each $S$, write
\[
\widehat{R}_j\coloneqq\begin{bmatrix}\widehat{U}_j\\\widehat{V}_j\end{bmatrix}.
\]
By Lemma~\ref{lem:regular}, $\|\widehat{T}_j\|_\star\le\frac12\bigl(\|\widehat{U}_j\|_F^2+\|\widehat{V}_j\|_F^2\bigr)$.
Since $P_j\succeq0$, $\langle P_j,\widehat{R}_j\widehat{R}_j^\top\rangle\ge0$.
Therefore
\[
\mathcal{L}(\{\widehat{T}_j\})+\sum_{j\in\mathcal{I}}\lambda_j\|\widehat{T}_j\|_\star\le\mathcal{L}^{\mathrm{lora}}_{\{\lambda_j\},\{P_j\}}(\{\widehat{U}_j,\widehat{V}_j\}).
\]
For each $S$, since $\rank(T_j^\star)\le r_j$, choose a balanced SVD factorization padded with zero columns if necessary, so that
\[
T_j^\star=U_j^\star V_j^{\star\top}\text{ and }\|T_j^\star\|_\star=\frac12\bigl(\|U_j^\star\|_F^2+\|V_j^\star\|_F^2\bigr).
\]
By global optimality of $\{(\widehat{U}_j,\widehat{V}_j)\}$ for the perturbed factorized objective,
\[
\mathcal{L}^{\mathrm{lora}}_{\{\lambda_j\},\{P_j\}}(\{\widehat{U}_j,\widehat{V}_j\})\le\mathcal{L}^{\mathrm{lora}}_{\{\lambda_j\},\{P_j\}}(\{U_j^\star,V_j^\star\}).
\]
The right-hand side equals
\[
\mathcal{L}(\{T_j^\star\})+\sum_{j\in\mathcal{I}}\lambda_j\|T_j^\star\|_\star+\frac12\sum_{j\in\mathcal{I}}\langle P_j,R_j^\star R_j^{\star\top}\rangle.
\]
Since $\{T_j^\star\}$ is a global minimizer of the nuclear-norm-regularized problem, combining the preceding displays proves the first inequality.

For the final bound, use $P_j\succeq0$, $\|P_j\|_2\le\varepsilon_j$, and
\[
\tr(R_j^\star R_j^{\star\top})=\|U_j^\star\|_F^2+\|V_j^\star\|_F^2=2\|T_j^\star\|_\star.
\]
Then
\[
\frac12\langle P_j,R_j^\star R_j^{\star\top}\rangle\le\frac12\|P_j\|_2\tr(R_j^\star R_j^{\star\top})\le\varepsilon_j\|T_j^\star\|_\star.
\]
Summing over $j\in\mathcal{I}$ gives the claimed perturbation bound.
\end{proof}

\section{Experiment details}\label{sec:experiment-details}

\begin{table}[htbp]
\centering
\caption{Hyperparameters for LoRA fine-tuning and full fine-tuning on NLP tasks.}
\label{tab:nlp_hyperparameters}

\begin{tabular}{llccc}
\toprule
\multicolumn{5}{l}{\textbf{LoRA fine-tuning}} \\
\midrule
Task/Dataset & & MNLI & RTE & SST-2 \\
\midrule
Batch size & & 8 & 8 & 8 \\
\multirow{2}{*}{Learning rate} 
& Head & 5e-4 & 1e-4 & 1e-4 \\
& Attention & 8e-5 & 1e-3 & 5e-4 \\
LoRA regularization $\lambda$ & & 1e-4 & 1e-4 & 1e-4 \\
\bottomrule
\end{tabular}

\vspace{0.6cm}

\begin{tabular}{lccc}
\toprule
\multicolumn{4}{l}{\textbf{Full fine-tuning}} \\
\midrule
Task/Dataset & MNLI & RTE & SST-2 \\
\midrule
Batch size & 8 & 8 & 8 \\
Learning rate & 2.32e-5 & 3.375e-4 & 2e-4 \\
Trace regularization $\lambda$ & 1e-4 & 1e-4 & 1e-4 \\
\bottomrule
\end{tabular}
\end{table}

\paragraph{Proxy estimation.}
The proxy rank is chosen large enough to exactly represent the selected head block; for Qwen2.5-1.5B, this is 128, matching the query-head dimension $d_k$.

The proxy is trained with the same regularized objective as the LoRA sweep, using $\lambda=10^{-4}$, AdamW learning rate $5\times 10^{-4}$, and best-checkpoint selection by the lowest training $\mathcal{L}^{\mathrm{lora}}_\lambda$.
After training, the LoRA update is merged into the base linear weights, and the model is run in evaluation mode over the training split.
We then estimate the downstream activation rank using $\rank(H)=\rank(HH^\top)$ to lower the computation by considering the rank of a smaller matrix.
The numerical rank is computed by the number of eigenvalues of $HH^\top$ that are larger than the eigenvalue threshold $10^{-4}$.

Let $\hat{r}_{\mathrm{up},S}$ denote the numerical rank of the upstream frozen-feature subspace for block $S$, and let $\hat{r}_{\mathrm{down},S}$ denote the numerical rank of the downstream observable subspace estimated after the trace-norm proxy optimization.
The final predicted threshold is
\[
r_{\mathrm{threshold}}=\max_S\min\{\hat{r}_{\mathrm{up},S},\hat{r}_{\mathrm{down},S}\}.
\]
This proxy approximates the optimizer-dependent quantity appearing in Theorem~\ref{thm:rank-attn}.

Before the LoRA sweep, the classification head is warmed up for 10 epochs, then frozen.
For the high-rank bound case, we use 100 samples, and for the low-rank bound case, we use 4 samples for training.
\newpage

\section*{NeurIPS Paper Checklist}

The checklist is designed to encourage best practices for responsible machine learning research, addressing issues of reproducibility, transparency, research ethics, and societal impact. Do not remove the checklist: {\bf The papers not including the checklist will be desk rejected.} The checklist should follow the references and follow the (optional) supplemental material.  The checklist does NOT count towards the page
limit. 

Please read the checklist guidelines carefully for information on how to answer these questions. For each question in the checklist:
\begin{itemize}
    \item You should answer \answerYes{}, \answerNo{}, or \answerNA{}.
    \item \answerNA{} means either that the question is Not Applicable for that particular paper or the relevant information is Not Available.
    \item Please provide a short (1--2 sentence) justification right after your answer (even for \answerNA). 
\end{itemize}

{\bf The checklist answers are an integral part of your paper submission.} They are visible to the reviewers, area chairs, senior area chairs, and ethics reviewers. You will also be asked to include it (after eventual revisions) with the final version of your paper, and its final version will be published with the paper.

The reviewers of your paper will be asked to use the checklist as one of the factors in their evaluation. While \answerYes{} is generally preferable to \answerNo{}, it is perfectly acceptable to answer \answerNo{} provided a proper justification is given (e.g., error bars are not reported because it would be too computationally expensive'' or ``we were unable to find the license for the dataset we used''). In general, answering \answerNo{} or \answerNA{} is not grounds for rejection. While the questions are phrased in a binary way, we acknowledge that the true answer is often more nuanced, so please just use your best judgment and write a justification to elaborate. All supporting evidence can appear either in the main paper or the supplemental material, provided in appendix. If you answer \answerYes{} to a question, in the justification please point to the section(s) where related material for the question can be found.

IMPORTANT, please:
\begin{itemize}
    \item {\bf Delete this instruction block, but keep the section heading ``NeurIPS Paper Checklist"},
    \item  {\bf Keep the checklist subsection headings, questions/answers and guidelines below.}
    \item {\bf Do not modify the questions and only use the provided macros for your answers}.
\end{itemize}


\begin{enumerate}

\item {\bf Claims}
    \item[] Question: Do the main claims made in the abstract and introduction accurately reflect the paper's contributions and scope?
    \item[] Answer: \answerYes{} 
    \item[] Justification: We provided technical proof and experimental evidence to support our claims.
    \item[] Guidelines:
    \begin{itemize}
        \item The answer \answerNA{} means that the abstract and introduction do not include the claims made in the paper.
        \item The abstract and/or introduction should clearly state the claims made, including the contributions made in the paper and important assumptions and limitations. A \answerNo{} or \answerNA{} answer to this question will not be perceived well by the reviewers. 
        \item The claims made should match theoretical and experimental results, and reflect how much the results can be expected to generalize to other settings. 
        \item It is fine to include aspirational goals as motivation, as long as it is clear that these goals are not attained by the paper. 
    \end{itemize}

\item {\bf Limitations}
    \item[] Question: Does the paper discuss the limitations of the work performed by the authors?
    \item[] Answer: \answerYes{} 
    \item[] Justification: We discuss the limitations in Section~\ref{sec:conclusion}.
    \item[] Guidelines:
    \begin{itemize}
        \item The answer \answerNA{} means that the paper has no limitation, while the answer \answerNo{} means that the paper has limitations, but those are not discussed in the paper. 
        \item The authors are encouraged to create a separate ``Limitations'' section in their paper.
        \item The paper should point out any strong assumptions and how robust the results are to violations of these assumptions (e.g., independence assumptions, noiseless settings, model well-specification, asymptotic approximations only holding locally). The authors should reflect on how these assumptions might be violated in practice and what the implications would be.
        \item The authors should reflect on the scope of the claims made, e.g., if the approach was only tested on a few datasets or with a few runs. In general, empirical results often depend on implicit assumptions, which should be articulated.
        \item The authors should reflect on the factors that influence the performance of the approach. For example, a facial recognition algorithm may perform poorly when image resolution is low or images are taken in low lighting. Or a speech-to-text system might not be used reliably to provide closed captions for online lectures because it fails to handle technical jargon.
        \item The authors should discuss the computational efficiency of the proposed algorithms and how they scale with dataset size.
        \item If applicable, the authors should discuss possible limitations of their approach to address problems of privacy and fairness.
        \item While the authors might fear that complete honesty about limitations might be used by reviewers as grounds for rejection, a worse outcome might be that reviewers discover limitations that aren't acknowledged in the paper. The authors should use their best judgment and recognize that individual actions in favor of transparency play an important role in developing norms that preserve the integrity of the community. Reviewers will be specifically instructed to not penalize honesty concerning limitations.
    \end{itemize}

\item {\bf Theory assumptions and proofs}
    \item[] Question: For each theoretical result, does the paper provide the full set of assumptions and a complete (and correct) proof?
    \item[] Answer: \answerYes{} 
    \item[] Justification: The paper provided the full set of assumptions and complete proofs in Section~\ref{sec:setup}, Section~\ref{sec:results-new}, and Appendix.

    \item {\bf Experimental result reproducibility}
    \item[] Question: Does the paper fully disclose all the information needed to reproduce the main experimental results of the paper to the extent that it affects the main claims and/or conclusions of the paper (regardless of whether the code and data are provided or not)?
    \item[] Answer: \answerYes{} 
    \item[] Justification: We provided experiment descriptions in the main paper and appendix.

\item {\bf Open access to data and code}
    \item[] Question: Does the paper provide open access to the data and code, with sufficient instructions to faithfully reproduce the main experimental results, as described in supplemental material?
    \item[] Answer: \answerNo{} 
    \item[] Justification: The experiments use GLUE tasks and an open-source Qwen2.5 checkpoint, but the current paper does not yet include released code.

\item {\bf Experimental setting/details}
    \item[] Question: Does the paper specify all the training and test details (e.g., data splits, hyperparameters, how they were chosen, type of optimizer) necessary to understand the results?
    \item[] Answer: \answerYes{} 
    \item[] Justification: We provided details for experiments in the paper.

\item {\bf Experiment statistical significance}
    \item[] Question: Does the paper report error bars suitably and correctly defined or other appropriate information about the statistical significance of the experiments?
    \item[] Answer: \answerNo{} 
    \item[] Justification: The current figures show single optimization trajectories without error bars.

\item {\bf Experiments compute resources}
    \item[] Question: For each experiment, does the paper provide sufficient information on the computer resources (type of compute workers, memory, time of execution) needed to reproduce the experiments?
    \item[] Answer: \answerNo{} 
    \item[] Justification: The draft does not report GPU type, memory, runtime per run, and total compute.
    
\item {\bf Code of ethics}
    \item[] Question: Does the research conducted in the paper conform, in every respect, with the NeurIPS Code of Ethics \url{https://neurips.cc/public/EthicsGuidelines}?
    \item[] Answer: \answerYes{} 
    \item[] Justification: Our paper follows the NeurIPS Code of Ethics.

\item {\bf Broader impacts}
    \item[] Question: Does the paper discuss both potential positive societal impacts and negative societal impacts of the work performed?
    \item[] Answer: \answerYes{} 
    \item[] Justification: This paper identifies several phenomena in general parameter-efficient fine-tuning methods that deepen our understanding of LoRA and may inform the development of more capable small models. Advances in these methods could have both positive and negative societal impacts.
    
\item {\bf Safeguards}
    \item[] Question: Does the paper describe safeguards that have been put in place for responsible release of data or models that have a high risk for misuse (e.g., pre-trained language models, image generators, or scraped datasets)?
    \item[] Answer: \answerNA{} 
    \item[] Justification: Our work does not pose such risks.

\item {\bf Licenses for existing assets}
    \item[] Question: Are the creators or original owners of assets (e.g., code, data, models), used in the paper, properly credited and are the license and terms of use explicitly mentioned and properly respected?
    \item[] Answer: \answerYes{} 
    \item[] Justification: We have properly cited papers and resources used in our experiment.

\item {\bf New assets}
    \item[] Question: Are new assets introduced in the paper well documented and is the documentation provided alongside the assets?
    \item[] Answer: \answerNA{} 
    \item[] Justification: Our paper does not release new assets.

\item {\bf Crowdsourcing and research with human subjects}
    \item[] Question: For crowdsourcing experiments and research with human subjects, does the paper include the full text of instructions given to participants and screenshots, if applicable, as well as details about compensation (if any)? 
    \item[] Answer: \answerNA{} 
    \item[] Justification: We do not have experiments involving crowdsourcing or research with human subjects.

\item {\bf Institutional review board (IRB) approvals or equivalent for research with human subjects}
    \item[] Question: Does the paper describe potential risks incurred by study participants, whether such risks were disclosed to the subjects, and whether Institutional Review Board (IRB) approvals (or an equivalent approval/review based on the requirements of your country or institution) were obtained?
    \item[] Answer: \answerNA{} 
    \item[] Justification: We do not involve crowdsourcing or research with human subjects.

\item {\bf Declaration of LLM usage}
    \item[] Question: Does the paper describe the usage of LLMs if it is an important, original, or non-standard component of the core methods in this research? Note that if the LLM is used only for writing, editing, or formatting purposes and does \emph{not} impact the core methodology, scientific rigor, or originality of the research, declaration is not required.
    \item[] Answer: \answerNA{} 
    \item[] Justification: LLMs have been employed solely for writing, editing, and formatting purposes. Their use does not affect the core methodology, scientific rigor, or originality of the research.

\end{enumerate}

\end{document}